\documentclass[10pt]{article}

\usepackage{graphicx}
\usepackage{color}

\usepackage[authoryear]{natbib}

\usepackage{bm}
\usepackage{amsmath}
\usepackage{amsfonts}
\usepackage{amsthm}
\usepackage{amssymb}
\usepackage{mathtools}
\mathtoolsset{showonlyrefs=true}
\newtheorem{theorem}{Theorem}
\newtheorem{lemma}{Lemma}
\newtheorem{corollary}{Corollary}

\newtheorem{remark}{Remark}
\DeclareMathOperator{\trace}{tr}
\DeclareMathOperator*{\argmin}{arg~min}
\DeclareMathOperator*{\argmax}{arg~max}

\usepackage{url}
\usepackage{subcaption}
\usepackage{algorithm}
\usepackage{algorithmic}

\makeatletter
\renewcommand{\ALG@name}{Procedure}
\makeatother

\usepackage[hidelinks]{hyperref}

\title{\vspace{-7mm}Bounded P-values in Parametric Programming-based Selective Inference}

\date{\today}

\makeatletter
\def\@fnsymbol#1{\ensuremath{\ifcase#1\or
{1}\or
{2}\or
{3}\or
{\dagger}\or
\else\@ctrerr\fi}}
\makeatother

\author{
Tomohiro Shiraishi\thanks{Nagoya University} ,
Daiki Miwa\thanks{Nagoya Institute of Technology} ,\\
Vo Nguyen Le Duy\thanks{RIKEN} ,
Ichiro Takeuchi\footnotemark[1] \footnotemark[3] \thanks{Corresponding author. e-mail: ichiro.takeuchi@mae.nagoya-u.ac.jp}
}
\begin{document}

\maketitle

\thispagestyle{empty}

\begin{abstract}
  \noindent
  Selective inference (SI) has been actively studied as a promising framework for statistical hypothesis testing for data-driven hypotheses.
The basic idea of SI is to make inferences conditional on an event that a hypothesis is selected.
In order to perform SI, this event must be characterized in a traceable form.
When selection event is too difficult to characterize, additional conditions are introduced for tractability.
This additional conditions often causes the loss of power, and this issue is referred to as over-conditioning in~\citep{fithian2014optimal}.
Parametric programming-based SI (PP-based SI) has been proposed as one way to address the over-conditioning issue.
The main problem of PP-based SI is its high computational cost due to the need to exhaustively explore the data space.
In this study, we introduce a procedure to reduce the computational cost while guaranteeing the desired precision, by proposing a method to compute the lower and upper bounds of $p$-values.
We also proposed three types of search strategies that efficiently improve these bounds.
We demonstrate the effectiveness of the proposed method in hypothesis testing problems for feature selection in linear models and attention region identification in deep neural networks.
%








\end{abstract}

\newpage
\section{Introduction}
\label{sec:introduction}

The advancement of measurement and information technologies has enabled the acquisition of large volumes of data across various scientific fields.
The approach, known as \emph{data-driven science}, seeks to uncover new discoveries through the analysis of scientific data.
Data-driven science is considered as a novel scientific paradigm, often referred to as \emph{the fourth paradigm}, in addition to the conventional three paradigms grounded in theory, experiment, and simulation.

In data-driven science, hypotheses are selected based on the data.
Therefore, the selected hypotheses are adapted to that data.
This creates a challenge when evaluating the hypotheses by statistical inference, such as computing $p$-values or confidence intervals, as it introduces a bias known as \emph{selection bias}~\citep{breiman1992little,potscher1991effects,leeb2005model,leeb2006can,benjamini2009selective,kriegeskorte2009circular,potscher2010confidence,benjamini2020selective}.
A common approach to address the selection bias issue is to split the data into two subsets: one for hypothesis selection and the other for evaluation.
However, data splitting may result in a loss of statistical power due to a smaller sample size, require the data to be independent and identically distributed, and prevent the evaluation of individual instance rather than the bulk evaluation of the entire dataset.

\emph{Selective inference (SI)} is attracting attention as a promising framework for statistical hypothesis testing for data-driven hypotheses (see, e.g., \citet{taylor2015statistical} for an overview).
The fundamental idea of SI introduced in a seminal paper by \citet{lee2016exact} is to perform statistical hypothesis testing by conditioning on an event that a hypothesis is selected (referred to as the \emph{hypothesis selection event} or \emph{selection event}).
In SI, by conditioning on the hypothesis selection event, it is possible to correct for the selection bias.
Initially, SI was studied in the context of feature selection for linear models (e.g., \citep{loftus2014significance,fithian2015selective,tibshirani2016exact,yang2016selective,hyun2018exact}).
Later, it has been extended to various problems
(e.g., \citep{rugamer2020inference,yamada2018post,suzumura2017selective,umezu2017selective,neufeld2022tree}).

In order to perform SI, it is necessary to characterize the event of selecting a hypothesis in a tractable form.
All the conditional SI studies listed above focus on the selection event which can be characterized as an intersection of linear and quadratic inequalities.
For complex selection events that cannot be easily characterized, additional conditions are introduced for computational tractability.
This is often referred to as \emph{over-conditioning (OC)}.
While SI with OC can still control the type I error rate below the significance level, it is known that the statistical power decreases due to OC (see \citet{fithian2014optimal} for theoretical discussion on OC and loss of powers).

To address the power loss caused by OC, several approaches have been proposed, including the randomized approach~\citep{tian2018selective,panigrahi2022exact}, the sampling approach~\citep{terada2017selective,panigrahi2022approximate}, and the parametric programming (PP) approach~\citep{liu2018more,le2021parametric}.
However, these approaches incur significant computational costs in exhaustively exploring the data space to identify the regions where the hypotheses of interest are selected.
In other words, these approaches require running ML algorithms for hypothesis selection many times on various datasets, which leads to high computational cost, especially when the hypothesis selection algorithm itself is computationally costly.

In this study, we focus on PP-based SI\footnote{
	PP-based SI is an approach for a common problem setting known as {\it saturated model setting}, while randomization and sampling approaches are mainly discussed in a more complex problem setting called {\it selected model setting} (see \citet{fithian2014optimal} for discussions on the difference of saturated and selected model settings).
	In this study, we restrict our discussion to the saturated model setting and do not discuss randomization approaches and sampling approaches.
}.
The basic idea of PP-based SI is based on the insights provided by \citet{liu2018more}.
In a common setting of SI called {\it saturated model setting}, SI is reduced to a one-dimensional line search problem in a high-dimensional data space defined by the test statistic (see section~\ref{sec:sec2} for details).
If we can solve the one-dimensional search problem, we can avoid OC issues and conduct SI without loss of power.
In \citet{liu2018more}, the authors discuss an approach to solve the one-dimensional search problem using a heuristic grid search.
However, there is no theoretical guarantee for the SI result based on a heuristic grid search because there is no performance guarantee between grid points.
To address this problem, \citet{le2021parametric} proposed incorporating a optimization technique called {\it parametric programming (PP)}~\citep{gal2012advances} into the SI framework.
PP is an optimization technique for solving a parametrized class of optimization problems and has been used, e.g., in regularization path computation~\citep{hastie2004entire, mairal2012complexity, tibshirani2011solution}.
PP-based SI can provide valid results if PP is applied to the entire one-dimensional line in the data space defined by the test statistic.
It has been demonstrated that PP-based SI is more powerful than existing SI approaches~\citep{sugiyama2021more,duy2022more}.
Furthermore, PP-based SI is applicable to new problems such as inferences on deep neural networks (DNNs)~\citep{duy2022quantifying, miwa2023valid}.

The goal of this study is to improve the PP-based SI in two key aspects.
The first aspect is on reducing the computational cost associated with PP-based SI, as the exhaustive search of one-dimensional line by PP is often time-consuming.
The second aspect is to establish a theoretical guarantee for the inferential results of PP-based SI, even when an exhaustive search is not conducted\footnote{
	It is important to note that a rigorous coverage of the infinite length of a one-dimensional line by PP is impossible.
	In the existing studies on PP-based SI mentioned earlier, a long yet finite line segment is considered for the one-dimensional search by PP.
}.
Specifically, we propose a method that enables us to obtain lower and upper bounds of $p$-values at any given point during the exploration in the data space.
This capability reduces the search cost for obtaining $p$-values with the desired precision.
Moreover, when it comes to hypothesis testing, where we only need to determine whether the null hypothesis is rejected or not, efficient decision-making can be achieved by utilizing these lower and upper bounds\footnote{
	Specifically, if the upper bound of the $p$-value is smaller than the significance level $\alpha$, the null hypothesis can be rejected.
	Conversely, if the lower bound is greater than or equal to $\alpha$, the null hypothesis cannot be rejected.
}.
To our knowledge, existing studies that discuss the bounds on selective $p$-values are limited to the works by \citet{jewell2022testing} and \citet{chen2023more}.
In these studies, the authors derived upper bounds for selective $p$-values in a specific case of a specific problem settings. As discussed later in section~\ref{subsec:bounding_selective_p} and Appendix~\ref{app:corollary_of_main_theorem}, their bounds are specific cases of the more general bounds provided by our proposed method.

Our contributions in this paper are summarized as follows.
\begin{itemize}

	\item
	      We propose a method to obtain the lower and upper bounds of $p$-values using the results of the data space exploration conducted up to any arbitrary point.
	      It enables us to reduce the computational cost required for computing $p$-values in PP-based SI.

	\item
	      We propose three options for data space search strategies.
	      By employing these strategies to explore the data space, it becomes possible to obtain $p$-values with the desired precision in fewer iterations.

	\item
	      We demonstrate the effectiveness of the proposed method by comparing it with the conventional method in hypothesis testing problems for feature selection in linear models and attention region identification in DNNs.

\end{itemize}

\paragraph{Organization of the paper}
The rest of the paper is organized as follows.
In section~\ref{sec:sec2}, we formulate the problem settings, and summarize existing SI approaches.
In section~\ref{sec:sec3}, we present how to derive the lower and upper bounds of $p$-values, which is the main contribution of this study.
Furthermore, we introduce three search strategies for exploring the data space in the PP-based SI.
In section~\ref{sec:sec4}, we conduct numerical experiments to compare the proposed method with conventional methods in two examples.
Finally, we discuss future works in section~\ref{sec:sec5}.
The code of the proposed method is available at \url{https://github.com/Takeuchi-Lab-SI-Group/sicore} and that of numerical experiments in section~\ref{sec:sec4} is available at \url{https://github.com/shirara1016/bounded_p_values_in_si}.

\newpage
\section{Existing Approaches to Selective Inference}
\label{sec:sec2}
\subsection{General Description}
%
%
%
In this section, we describe the general framework of Selective Inference (SI).
We first describe the problem setting of SI we considered in this paper.
%
%
Then, we describe the SI with over-conditioning (OC).
When the OC approach is used in decision making, the type I error rate can be properly controlled at the designated significance level, but the power is usually lower than the case without OC.
Finally, we describe the parametric programming-based SI (PP-based SI) to address the issue of OC.
%
%
%
%
%
See \citep{le2021parametric,sugiyama2021more,duy2022quantifying,miwa2023valid} for more details on the contents in this section.

\paragraph{Notations.}
To formulate the problem, we consider a random data vector
\begin{equation}
	\bm{D} = (D_1, ..., D_n)^\top \sim \mathcal{N}({\bm{\mu}}, \sigma^2I_n),
\end{equation}
where $n$ is the number of instances, ${\bm{\mu}}$ is an unknown vector, and $\sigma^2I_n \in \mathbb{R}^{n \times n}$ is a covariance matrix that is known or estimable from independent data.
The goal is to quantify the statistical significance of the data-driven hypotheses that are obtained by applying a ML algorithm to the data $\bm{D}$.
\paragraph{Statistical inference.}
First, we apply a ML algorithm $\mathcal{A}\colon \bm{D}\mapsto \mathcal{A}(\bm{D})$ to the observed data $\bm{D}^\mathrm{obs}$ and consider some statistical test problems based on the result $\mathcal{A}(\bm{D}^\mathrm{obs})$.
We consider the \emph{selective $z$-test}:
\begin{equation}
	\mathrm{H}_{0}: \bm{\eta}^\top \bm{\mu} = 0
	\quad \text{v.s.} \quad
	\mathrm{H}_{1}: \bm{\eta}^\top \bm{\mu} \neq 0,
\end{equation}
and the \emph{selective $\chi$-test}:
\begin{equation}
	\mathrm{H}_{0}: \left\|P\bm{\mu}\right\| = 0
	\quad \text{v.s.} \quad
	\mathrm{H}_{1}: \left\| P \bm{\mu}\right\| \neq 0,
\end{equation}
where $\bm{\eta}\in\mathbb{R}^n$ is a vector and $P\in\mathbb{R}^{n\times n}$ is a projection matrix that depends on $\mathcal{A}(\bm{D}^\mathrm{obs})$.
\paragraph{Test-statistic.}
For each of these hypothesis tests, we consider test statistics called \emph{conditionally $z$ test-statistic}, which is expressed as
\begin{equation}
	T(\bm{D}) = \bm{\eta}^\top \bm{D},
\end{equation}
and \emph{conditionally $\chi$ test-statistic}, which is expressed as
\begin{equation}
	T(\bm{D}) = \left\|P\bm{D}\right\| / \sigma,
\end{equation}
respectively.
\paragraph{Conditional selective inference}
To conduct the statistical tests, we consider the test statistic conditional on the ML algorithm event, i.e.,
\begin{equation}
	\label{eq:conditonal_sampling_distribution}
	T(\bm{D}) \mid
	\left\{
	\mathcal{A}(\bm{D}) = \mathcal{A}(\bm{D}^\mathrm{obs})
	\right\}.
\end{equation}
To compute the selective $p$-value based on the conditional sampling distribution in~\eqref{eq:conditonal_sampling_distribution}, we need to additionally condition on the nuisance component $\mathcal{Q}(\bm{D})$, where
\begin{equation}
	\mathcal{Q}(\bm{D})
	=
	(I_n - \bm{\eta}\bm{\eta}^\top / \left\|\bm{\eta}\right\|^2)\bm{D}
	\in \mathbb{R}^n
\end{equation}
for the selective $z$-test, and
\begin{equation}
	\mathcal{Q}(\bm{D})
	=
	\left(
	\sigma P\bm{D} / \left\|P\bm{D}\right\|, P^\perp D
	\right)
	\in \mathbb{R}^{n} \times \mathbb{R}^{n}
\end{equation}
for the selective $\chi$-test\footnote{
	Note that conditioning on the nuisance component can also be interpreted as an OC and that the power of SI can be improved if we conduct SI without this OC.
	Unfortunately, however, PP-based SI and other existing SI approaches cannot avoid OC on the nuisance component.
	This is because, without conditioning on the nuisance component, we need to compute the truncated normal distribution in an $n$-dimensional space, which is computationally quite challenging.
}.
The selective $p$-value is then computed as:
\begin{equation}
	\label{eq:general_selective_p_value}
	p^\mathrm{selective}
	=
	\mathbb{P}_{\mathrm{H}_0}
	\left(
	|T(\bm{D})| \geq |T(\bm{D}^\mathrm{obs})|
	~\Big|~
	\bm{D} \in \mathcal{D}
	\right),
\end{equation}
where
$
	\mathcal{D}=
	\left\{
	\bm{D}\in\mathbb{R}^n:
	\mathcal{A}(\bm{D}) = \mathcal{A}(\bm{D}^\mathrm{obs}),
	\mathcal{Q}(\bm{D}) = \mathcal{Q}(\bm{D}^\mathrm{obs})
	\right\}
$.
\paragraph{Characterization of the conditional data space.}
The data space $\mathbb{R}^n$ conditional on $\mathcal{A}(\bm{D}) = \mathcal{A}(\bm{D}^\mathrm{obs})$ is represented as a subset of $\mathbb{R}^n$.
Furthermore, by conditioning also on the nuisance component $\mathcal{Q}(\bm{D}) = \mathcal{Q}(\bm{D}^\mathrm{obs})$, the subspace $\mathcal{D}$ in~\eqref{eq:general_selective_p_value} is a one-dimensional line in $\mathbb{R}^n$.
Therefore, the set $\mathcal{D}$ can be re-written, using a scalar parameter $z\in\mathbb{R}$, as
\begin{equation}
	\mathcal{D}
	=
	\left\{
	\bm{D}(z)\in\mathbb{R}^n:
	\bm{D}(z) = \bm{a} + \bm{b}z, z\in\mathcal{Z}
	\right\},
\end{equation}
using vectors $\bm{a}, \bm{b}\in\mathbb{R}^n$ defined as
\begin{equation}
	\label{eq:direction_of_z_test}
	\bm{a} = (I_n - \bm{\eta}\bm{\eta}^\top/\left\|\bm{\eta}\right\|^2)\bm{D}^\mathrm{obs},
	\quad
	\bm{b} = \bm{\eta} / \left\|\bm{\eta}\right\|^2
\end{equation}
for the selective $z$-test, and
\begin{equation}
	\label{eq:direction_of_chi_test}
	\bm{a} = P^\perp \bm{D}^\mathrm{obs},
	\quad
	\bm{b} = \sigma P\bm{D}^\mathrm{obs} / \left\|P\bm{D}^\mathrm{obs}\right\|
\end{equation}
for the selective $\chi$-test, and
\begin{equation}
	\label{eq:Z_region}
	\mathcal{Z} =
	\left\{
	z\in\mathbb{R}:
	\mathcal{A}(\bm{a} + \bm{b}z) = \mathcal{A}(\bm{D}^\mathrm{obs})
	\right\}.
\end{equation}
Let us consider a random variable $Z\in\mathbb{R}$ and its observation $Z^\mathrm{obs}\in\mathbb{R}$ such that they respectively satisfy $\bm{D}=\bm{a}+\bm{b}Z$ and $\bm{D}^\mathrm{obs}=\bm{a}+\bm{b}Z^\mathrm{obs}$.
The (two-tailed) selective $p$-value in~\eqref{eq:general_selective_p_value} is re-written as
\begin{equation}
	\label{eq:one_dimensional_selective_p_value}
	p^\mathrm{selective}
	=
	\mathbb{P}_{\mathrm{H}_0}
	\left(
	|Z| > |Z^\mathrm{obs}|
	~\Big|~
	Z \in \mathcal{Z}
	\right).
\end{equation}
Because the variable $Z\sim\mathcal{N}(0, \sigma^2\left\|\bm{\eta}\right\|^2)$ for the selective $z$-test and $Z\sim\chi(\trace P)$ for the selective $\chi$-test under the null hypothesis, $Z\mid Z\in\mathcal{Z}$ follows a \emph{truncated normal distribution} and a \emph{truncated $\chi$ distribution}, respectively.
Once the truncation region $\mathcal{Z}$ is identified, the selective $p$-value in~\eqref{eq:one_dimensional_selective_p_value} can be easily computed.
Thus, the remaining task is reduced to the characterization of $\mathcal{Z}$.
To identify $\mathcal{Z}$, we first introduce the concept of the OC and then use parametric programming to address OC issue.
\subsection{Over-Conditioning Approach}
%
%
In many ML algorithms, conditioning event only on the result of algorithm i.e., $\mathcal{A}(\bm{D})=\mathcal{A}(\bm{D}^\mathrm{obs})$ is often too \emph{complicated} to characterize.
Therefore, we introduce a \emph{sub-algorithm} $\mathcal{S}: \bm{D} \mapsto \mathcal{S}(\bm{D})$ such that the intersection of the event on the ML algorithm $\mathcal{A}(\bm{D}) = \mathcal{A}(\bm{D}^\mathrm{obs})$ and the event on the sub-algorithm $\mathcal{S}(\bm{D}) = \mathcal{S}(\bm{D}^\mathrm{obs})$ can be characterized in a tractable form.
This means that instead of~\eqref{eq:conditonal_sampling_distribution} we consider the following test statistic conditional on the ML algorithm event and the sub-algorithm event, i.e.,
\begin{equation}
	\label{eq:oc_conditonal_sampling_distribution}
	T(\bm{D}) \mid
	\left\{
	\mathcal{A}(\bm{D}) = \mathcal{A}(\bm{D}^\mathrm{obs}),
	\mathcal{S}(\bm{D}) = \mathcal{S}(\bm{D}^\mathrm{obs})
	\right\}.
\end{equation}
\paragraph{The OC conditional data space.}
For an arbitrary data $\bm{d} \in\mathbb{R}^n$, we define the OC conditional data space as:
\begin{equation}
	\mathcal{D}^\mathrm{oc}(\bm{d})
	=
	\left\{
	\bm{D} \in \mathbb{R}^n:
	\mathcal{A}(\bm{D}) = \mathcal{A}(\bm{d}),
	\mathcal{S}(\bm{D}) = \mathcal{S}(\bm{d})
	\right\}.
\end{equation}
Note that the case of $\bm{d}=\bm{D}^\mathrm{obs}$ corresponds to the conditioning in~\eqref{eq:oc_conditonal_sampling_distribution} and then we consider the $\mathcal{Z}\cap \mathcal{D}^\mathrm{oc}(\bm{D}^\mathrm{obs})$ instead of $\mathcal{Z}$ in~\eqref{eq:one_dimensional_selective_p_value} to compute the OC $p$-value.
In many ML algorithms, by introducing additional event on an appropriate sub-algorithm $\mathcal{S}$, the OC conditional data space can be characterized by an intersection of a set of \emph{linear/quadratic inequalities} of data $\bm{D}\in\mathbb{R}^n$.
In other words, we can introduce a sub-algorithm $\mathcal{S}$ such that when the original selection event on the ML algorithm $\mathcal{A}(\bm{D}) = \mathcal{A}(\bm{d})$ is combined with an additional selection event on the sub-algorithm $\mathcal{S}(\bm{D}) = \mathcal{S}(\bm{d})$, the data space conditional on the intersection of these two events can be characterized by an intersection of linear/quadratic inequalities of the data.
In this case, the OC conditional data space is written as:
\begin{align}
	\mathcal{D}^\mathrm{oc}(\bm{d})
	 & =
	\left\{
	\bm{D} \in \mathbb{R}^n:
	\mathcal{A}(\bm{D}) = \mathcal{A}(\bm{d}),
	\mathcal{S}(\bm{D}) = \mathcal{S}(\bm{d})
	\right\} \\
	\label{eq:general_oc_data_space}
	 & =
	\left\{
	\bm{D} \in \mathbb{R}^n:
	\bigcap_{t\in[n_t(\mathcal{A}(\bm{d}),\mathcal{S}(\bm{d}))]}
	g_t^{\mathcal{A}(\bm{d}),\mathcal{S}(\bm{d})}(\bm{D}) \leq 0
	\right\},
\end{align}
where $n_t(\mathcal{A}(\bm{d}),\mathcal{S}(\bm{d}))$ is the number of linear/quadratic inequalities
and
$g_{t}^{\mathcal{A}(\bm{d}),\mathcal{S}(\bm{d})}(\bm{D}) \leq 0$ represents the $t$-th inequality for a vector variable that depends on $\mathcal{A}(\bm{d})$ and $\mathcal{S}(\bm{d})$.
%
%
%
\subsection{Parametric Programming Approach}
\label{subsec:pp_based_si}
Since we only consider a line in the data space in~\eqref{eq:Z_region}, with a slight abuse of notation, we interpret that a real number $z$ is an input to the ML algorithm $\mathcal{A}$ and the sub-algorithm $\mathcal{S}$.
Namely, we write $\mathcal{A}(z)=\mathcal{A}(\bm{a}+\bm{b}z)$ and $\mathcal{S}(z)=\mathcal{S}(\bm{a}+\bm{b}z)$.
In order to obtain the region $\mathcal{Z}$ in~\eqref{eq:Z_region}, the idea of PP-based SI is to enumerate all the OC regions in
\begin{equation}
	\mathcal{Z}^\mathrm{oc}(\bm{a}+\bm{b}z)
	=
	\left\{
	r\in\mathbb{R}:
	\mathcal{A}(r) = \mathcal{A}(z),
	\mathcal{S}(r) = \mathcal{S}(z)
	\right\},
\end{equation}
and consider their union
\begin{equation}
	\label{eq:region_homotopy}
	\mathcal{Z}
	=
	\bigcup_{z\in\mathbb{R}\mid \mathcal{A}(z)=\mathcal{A}^\mathrm{obs}}
	\mathcal{Z}^\mathrm{oc}(\bm{a}+\bm{b}z)
\end{equation}
by parametric programming where we denote $\mathcal{A}(\bm{D}^\mathrm{obs})$ by $\mathcal{A}^\mathrm{obs}$ for short.
From~\eqref{eq:general_oc_data_space}, $\mathcal{Z}^\mathrm{oc}(\bm{a}+\bm{b}z)$ is written as:
\begin{align}
	\mathcal{Z}^\mathrm{oc}(\bm{a}+\bm{b}z)
	 & =
	\left\{
	r\in\mathbb{R}:
	\mathcal{A}(r) = \mathcal{A}(z),
	\mathcal{S}(r) = \mathcal{S}(z)
	\right\} \\
	 & =
	\left\{
	r\in\mathbb{R}:
	\bm{a}+\bm{b}r \in \mathcal{D}^\mathrm{oc}(\bm{a}+\bm{b}z)
	\right\} \\
	\label{eq:decomp_Z_oc}
	 & =
	\bigcap_{t\in[n_t(\mathcal{A}(z),\mathcal{S}(z))]}
	\left\{
	r\in \mathbb{R}:
	g_t^{(\mathcal{A}(z),\mathcal{S}(z))}(\bm{a}+\bm{b}r) \leq 0
	\right\}.
\end{align}
In addition, since each $g_t^{(\mathcal{A}(z),\mathcal{S}(z))}(\bm{a}+\bm{b}r) \leq 0$ is a linear/quadratic inequality for a vector variable, it is transformed into the intersection of the multiple linear/quadratic inequalities for a scalar variable $r$.
By analytically solving these linear/quadratic inequalities, an intersection of linear/quadratic inequalities for a scalar variable $r$ is represented as a union of multiple non-overlapping intervals, which we denote by:
\begin{equation}
	\label{eq:region_oc_final_form}
	\mathcal{Z}^\mathrm{oc}(\bm{a}+\bm{b}z)
	=
	\bigcup_{v \in [n_v(\mathcal{A}(z),\mathcal{S}(z))]}
	[\ell_{v}^{(\mathcal{A}(z),\mathcal{S}(z))}, u_{v}^{(\mathcal{A}(z),\mathcal{S}(z))}],
\end{equation}
where
$n_v(\mathcal{A}(z),\mathcal{S}(z))$
is the number of intervals, while
$\ell_{v}^{(\mathcal{A}(z),\mathcal{S}(z))}$
and
$u_{v}^{(\mathcal{A}(z),\mathcal{S}(z))}$
are the lower and the upper bound of the $v$-th interval that depend on $\mathcal{A}(z)$ and $\mathcal{S}(z)$.
%
%
Using~\eqref{eq:region_homotopy}, \eqref{eq:decomp_Z_oc} and \eqref{eq:region_oc_final_form}, the region $\mathcal{Z}$ in~\eqref{eq:Z_region} is given by:
\begin{align}
	\label{eq:completion_of_Z}
	\mathcal{Z}
	 & =
	\bigcup_{z\in \mathbb{R} \mid \mathcal{A}(z)=\mathcal{A}^\mathrm{obs}}
	\bigcap_{s\in [n_s(\mathcal{A}(z),\mathcal{S}(z))]}
	\left\{
	r\in \mathbb{R}:
	g_s^{(\mathcal{A}(z),\mathcal{S}(z))}(\bm{a}+\bm{b}r) \leq 0
	\right\} \\
	 & =
	\bigcup_{z\in \mathbb{R} \mid \mathcal{A}(z)=\mathcal{A}^\mathrm{obs}}
	\bigcup_{v\in [n_v(\mathcal{A}(z),\mathcal{S}(z))]}
	[\ell_{v}^{(\mathcal{A}(z),\mathcal{S}(z))}, u_{v}^{(\mathcal{A}(z),\mathcal{S}(z))}].
\end{align}
This enables us to compute selective $p$-values in \eqref{eq:one_dimensional_selective_p_value}.

\newpage
\section{Proposed method}
\label{sec:sec3}
To compute the selective $p$-value, we need to identify a set of $z \in \mathbb{R}$ that satisfies the conditioning part of~\eqref{eq:Z_region}.
By applying the ML algorithm $\mathcal{A}$ and the sub-algorithm $\mathcal{S}$ to the data $\bm{D}=\bm{a}+\bm{b}z$, which corresponds to the point $z \in \mathbb{R}$,
it is possible to obtain the $\mathcal{A}(z)$ and $\mathcal{S}(z)$.
Then, the intersection of linear/quadratic inequalities
\begin{equation}
	\bigcap_{t \in [n_t(\mathcal{A}(z),\mathcal{S}(z))]}
	\{r \in \mathbb{R}: g_t^{(\mathcal{A}(z),\mathcal{S}(z))}(\bm{a}+\bm{b}r) \leq 0\}
\end{equation}
are satisfied at the point $z$, and henceforth the union of corresponding intervals
\begin{equation}
	\label{eq:union_of_intervals_for_z}
	\bigcup_{v \in [n_v(\mathcal{A}(z),\mathcal{S}(z))]}
	[\ell_{v}^{(\mathcal{A}(z),\mathcal{S}(z))}, u_{v}^{(\mathcal{A}(z),\mathcal{S}(z))}]
\end{equation}
is obtained by analytically solving these inequalities.
By checking whether $\mathcal{A}(z)=\mathcal{A}^\mathrm{obs}$ or not, it is possible to know whether the union of intervals in~\eqref{eq:union_of_intervals_for_z} satisfies the condition of the region $\mathcal{Z}$ in~\eqref{eq:Z_region}.

The selective $p$-value $p^\mathrm{selective}$ is obtained by repeating such a process until the entire real line $\mathbb{R}$ is covered by exhaustively checking all the intervals.
However, there are often cases where a large number of intervals need to be checked, and it may require a too much computational cost to cover the entire real line.

In previous studies on PP-based SI~\citep{le2021parametric,sugiyama2021more,duy2022quantifying,miwa2023valid}, a wide range of exploration was conducted to exhaustively search along the line in the data space.
For examples, the interval ranging from $-20\sigma\|\bm{\eta}\|$ to $20\sigma\|\bm{\eta}\|$ was searched for the selective $z$-test, where the unconditional sampling distribution of $Z$ in~\eqref{eq:one_dimensional_selective_p_value} is $\mathcal{N}(0,\sigma^2\|\bm{\eta}\|^2)$.
Furthermore, the interval ranging from $0$ to $100$ was searched for the selective $\chi$-test, where the unconditional sampling distribution of $Z$ in~\eqref{eq:one_dimensional_selective_p_value} is $\chi(\trace P)$.
We refer to the PP-based SI based on exhaustive search as {\tt exhaustive}.
In this study, we propose a procedure to compute selective $p$-values based on PP-based SI \emph{without} the need for such exhaustive search.
Our proposed method allows for providing the lower and upper bounds of the selective $p$-value based on the intervals checked so far.
By setting the desired precision or significance level in advance, it becomes possible to obtain a sufficiently precise selective $p$-value and make correct decision without the need for conducting an exhaustive search.
\subsection{PP-based SI Procedure}
Procedure~\ref{alg:simple} shows how to update the lower and upper bounds of the selective $p$-value.
To explain the procedure, we introduce some notations.
Let $i$ be the index of the iteration, $S_i \subseteq \mathbb{R}$ be the set of intervals checked by the $i$-th iteration, and $R_i \subseteq S_i$ be a subset of $S_i$ such that the same result as the observed one is obtained when the ML algorithm is applied to the data $\bm{D}=\bm{a}+\bm{b}z$ for any $z \in R_i$.
We call $S_i$ and $R_i$ as \emph{searched intervals} and \emph{truncated intervals}, respectively.

Given the observed data $\bm{D}^\mathrm{obs}$ and the ML algorithm $\mathcal{A}$, we first compute the vectors $\bm{a}$ and $\bm{b}$ that characterize the direction of the test-statistic based on~\eqref{eq:direction_of_z_test} for the selective $z$-test and on~\eqref{eq:direction_of_chi_test} for the selective $\chi$-test, respectively.
Then, the point $Z^\mathrm{obs}$ that corresponds to the observed data $\bm{D}^\mathrm{obs}$ in such a way that $\bm{D}^\mathrm{obs}=\bm{a}+\bm{b}Z^\mathrm{obs}$ is obtained.
By applying the ML algorithm $\mathcal{A}$ and the sub-algorithm $\mathcal{S}$ at the observed point $Z^\mathrm{obs}$, the pair of $(\mathcal{A}^\mathrm{obs}, \mathcal{S}^\mathrm{obs})$ is obtained.
Then, the searched intervals and truncated intervals are initialized as:
\begin{equation}
	S_1 = R_1 =
	\bigcup_{v \in [n_v(\mathcal{A}^\mathrm{obs},\mathcal{S}^\mathrm{obs})]}
	[
	\ell_{v}^{(\mathcal{A}^\mathrm{obs},\mathcal{S}^\mathrm{obs})},
	u_{v}^{(\mathcal{A}^\mathrm{obs},\mathcal{S}^\mathrm{obs})}
	].
\end{equation}
In each iteration, a point $z$ is selected from unchecked subset $S_i^c \coloneq \mathbb{R} \setminus S_i$ (the specific strategies for selecting $z \in S_i^c$ will be discussed in section~\ref{subsec:search_strategy}).
Then, by applying the ML algorithm $\mathcal{A}$ and the sub-algorithm $\mathcal{S}$ at the selected point $z$, we obtain the union of intervals
\begin{equation}
	\bigcup_{v \in [n_v(\mathcal{A}(z),\mathcal{S}(z))]}
	[\ell_{v}^{(\mathcal{A}(z),\mathcal{S}(z))}, u_{v}^{(\mathcal{A}(z),\mathcal{S}(z))}].
\end{equation}
Note that a single point $z$ is selected and the OC region for this point is obtained in each iteration.
The searched intervals $S_i$ is updated by adding this union of intervals, while truncated intervals $R_i$ is updated similarly only if $\mathcal{A}(z) = \mathcal{A}^\mathrm{obs}$.
Using the searched intervals $S_i$ and the truncated intervals $R_i$, the lower and upper bounds of the selective $p$-value are computed by using functions $L$ and $U$ whose details are described in Theorem~\ref{thm:main} below.

We consider two termination criteria for Procedure~\ref{alg:simple} for two different purposes.
The first case is when we want to obtain the selective $p$-value with the desired precision $\varepsilon$.
In this case, we can terminate the procedure when the difference between the lower and upper bounds of the selective $p$-value becomes $U_i - L_i < \varepsilon$.
The second case is when we want to decide whether the null hypothesis is rejected or not for a specific significance level $\alpha$ (e.g., $\alpha = 0.05$).
In this case, it is possible to know that the null hypothesis is rejected if $L_i \geq \alpha$, whereas it is not rejected if $U_i \leq \alpha$.
This means that we can terminate the procedure when either of these conditions occurs.
We call the proposed method with the first termination criterion as \texttt{proposed(precision)} and that with the second termination criterion as \texttt{proposed(decision)}

\begin{algorithm}
	\caption{Procedure for computing the bounds of the selective $p$-value}
	\label{alg:simple}
	\begin{algorithmic}[1]
		\REQUIRE Observed data $\bm{D}^\mathrm{obs}$, the ML algorithm $\mathcal{A}$ and the sub-algorithm $\mathcal{S}$
		\STATE Compute $\bm{a}$ and $\bm{b}$, by~\eqref{eq:direction_of_z_test} or \eqref{eq:direction_of_chi_test}
		\STATE Obtain $Z^\mathrm{obs}$ such that $\bm{D}^\mathrm{obs}=\bm{a}+\bm{b}Z^\mathrm{obs}$
		\STATE $\mathcal{A}^\mathrm{obs} \leftarrow \mathcal{A}(Z^\mathrm{obs})$, $\mathcal{S}^\mathrm{obs} \leftarrow \mathcal{S}(Z^\mathrm{obs})$
		\STATE $S_1 \leftarrow R_1 \leftarrow
			\cup_{v \in [n_v(\mathcal{A}^\mathrm{obs},\mathcal{S}^\mathrm{obs})]}
			[
			\ell_{v}^{(\mathcal{A}^\mathrm{obs},\mathcal{S}^\mathrm{obs})},
			u_{v}^{(\mathcal{A}^\mathrm{obs},\mathcal{S}^\mathrm{obs})}
			]$
		\STATE $i\leftarrow 1$
		\WHILE{termination criterion is satisfied}
		\STATE Select a $z \in S_i^c$
		\STATE Compute $\mathcal{A}(z)$ and $\mathcal{S}(z)$
		\STATE $S_{i+1} \leftarrow
			S_i \cup \left\{
			\cup_{v \in [n_v(\mathcal{A}(z),\mathcal{S}(z))]}
			[\ell_{v}^{(\mathcal{A}(z),\mathcal{S}(z))}, u_{v}^{(\mathcal{A}(z),\mathcal{S}(z))}]
			\right\}$
		\IF{$\mathcal{A}(z) = \mathcal{A}^\mathrm{obs}$}
		\STATE $R_{i+1} \leftarrow
			R_i \cup \left\{
			\cup_{v \in [n_v(\mathcal{A}(z),\mathcal{S}(z))]}
			[\ell_{v}^{(\mathcal{A}(z),\mathcal{S}(z))}, u_{v}^{(\mathcal{A}(z),\mathcal{S}(z))}]
			\right\}$
		\ENDIF
		\STATE $i \leftarrow i+1$
		\STATE $L_i\leftarrow L(S_i, R_i)$
		\STATE $U_i\leftarrow U(S_i, R_i)$
		\ENDWHILE
		\ENSURE $[L_i, U_i]$
	\end{algorithmic}
\end{algorithm}

\subsection{Lower and Upper Bounds of Selective $p$-values}
\label{subsec:bounding_selective_p}
\paragraph{Main theorem.}
The following theorem states that, given a searched interval $S_i$ and truncated interval $R_i$, it is possible to obtain the lower and upper bounds of the selective $p$-value (lines 14-15 in Procedure~\ref{alg:simple}).
This theorem contains the technique of considering conservative (upper bound) $p$-values by adding the tails to truncated intervals $R_i$ for sufficiently large searched intervals $S_i$ in~\citep{jewell2022testing,chen2023more} as a special case.
We show this technique in Appendix~\ref{app:corollary_of_main_theorem} as a corollary of our theorem.
In this section, we only focus on the (two-sided) selective $p$-value in~\eqref{eq:one_dimensional_selective_p_value} for simplicity.
However, the results in this section can be easily extended to the case of the one-sided test (see Appendix~\ref{app:extension_to_one_sided} for more details).
\begin{theorem}
	\label{thm:main}
	Let $t$ be the observed test statistic, $f$ be the unconditional probability density function of the test statistic, $\{S_i\}_{i \in \mathbb{N}}, \{R_i\}_{i \in \mathbb{N}}$ be the sequences of subsets of $\mathcal{B}(\mathbb{R})$
	as defined in Procedure~\ref{alg:simple},
	where
	$\mathcal{B}(\mathbb{R})$
	is the borel set of $\mathbb{R}$, i.e., a set of measurable subsets of $\mathbb{R}$.
	%
	%
	Then,
	given
	$S_i$ and $R_i$,
	the lower and the upper bounds of the selective $p$-value are given by
	\begin{align}
		L_i & = L(S_i, R_i) =
		\frac{\mathcal{I}(R_i \setminus [-|t|,|t|])}
		{\mathcal{I}(R_i\cup (S_i^c\cap [-|t|,|t|]))},
		\\
		U_i & = U(S_i, R_i) =
		\frac{\mathcal{I}((R_i\cup S_i^c)\setminus [-|t|, |t|])}
		{\mathcal{I}(R_i\cup (S_i^c\setminus [-|t|,|t|]))},
	\end{align}
	where
	$\mathcal{I}\colon \mathcal{B}(\mathbb{R})\ni B\mapsto \int_B f(z)dz\in[0,1]$.
\end{theorem}
The proof of Theorem~\ref{thm:main} is presented in Appendix~\ref{app:proof_of_main_theorem}.
The outline of the proof is as follows.
From the Procedure~\ref{alg:simple}, we have $R_i\subseteq R_\infty\subseteq R_i\cup S_i^c$ for any integer $i$ and $R_\infty = \mathcal{Z}$ in~\eqref{eq:one_dimensional_selective_p_value}.
Then, we re-write the selective $p$-value in~\eqref{eq:one_dimensional_selective_p_value} as
\begin{align}
	p^\mathrm{selective}
	 & =
	\mathbb{P}_{\mathrm{H}_0}
	\left(
	|Z| \geq |t|
	~\Big|~
	Z \in R_\infty
	\right) \\
	 & =
	\left.
	\int_{R_\infty \setminus [-|t|,|t|]}f(z)dz
	\middle/
	\int_{R_\infty}f(z)dz
	\right. \\
	 & =
	\frac{\mathcal{I}(R_\infty\setminus [-|t|,|t|])}{\mathcal{I}(R_\infty)}.
\end{align}
Therefore, to evaluate the lower and upper bounds of the selective $p$-value, it is sufficient to solve the following two optimization problems:
\begin{equation}
	\inf_{R\in\mathcal{B}(\mathbb{R})\mid R_i\subseteq R\subseteq R_i\cup S_i^c}
	\frac{\mathcal{I}(R\setminus [-|t|,|t|])}{\mathcal{I}(R)},\
	\sup_{R\in\mathcal{B}(\mathbb{R})\mid R_i\subseteq R\subseteq R_i\cup S_i^c}
	\frac{\mathcal{I}(R\setminus [-|t|,|t|])}{\mathcal{I}(R)}.
\end{equation}
The term to be optimized is determined from the ratio of the probability masses outside and inside of $[-|t|,|t|]$.
The main idea for optimization is that we consider the two sets:
the first is $R_i$ with the intersection of $S_i^c$ and $[-|t|,|t|]$ to construct the lower bound, and the second is $R_i$ with $S_i^c$ minus $[-|t|,|t|]$ to construct the upper bound.
\paragraph{Properties of the bounds.}
The lower bounds $L_i$ and the upper bounds $U_i$ of the selective $p$-value obtained by Procedure~\ref{alg:simple} have some reasonable properties described in the following lemma.
\begin{lemma}
	\label{lemm:bounds_property}
	The lower bounds $L_i$ and upper bounds $U_i$ of the selective $p$-value obtained by Procedure~\ref{alg:simple} are monotonically increasing and decreasing, in $i$, respectively.
	Furthermore, each of them converges to the selective $p$-value as $i\to\infty$.
\end{lemma}
The proof of Lemma~\ref{lemm:bounds_property} is presented in Appendix~\ref{app:proof_of_bounds_property}.
Lemma~\ref{lemm:bounds_property} indicates that, in Procedure~\ref{alg:simple}, as the iteration progresses, the bounds of the selective $p$-value become tighter and converge to the selective $p$-value.
\subsection{Search Strategies}
\label{subsec:search_strategy}

From Lemma~\ref{lemm:bounds_property}, it is guaranteed that Procedure~\ref{alg:simple} will eventually terminate no matter how the search is conducted (i.e., no matter how the point $z$ is selected from $S_i^c$ as in line 7 of Procedure~\ref{alg:simple}).
In practice, however, the actual computational cost depends on how $z$ is selected from $S_i^c$.
%
Therefore, we discuss the search strategy, i.e., the selection method of $z$ from $S_i^c$ to construct $S_{i+1}$ from $S_i$.
Because Procedure~\ref{alg:simple} sequentially computes the lower and upper bounds of the selective $p$-value by using $S_i$ and $R_i$, using the notation of Theorem~\ref{thm:main}, our goal is to construct $S_{i+1}$ from $S_i$ such that the $U_{i+1}-L_{i+1}$ is minimized.
The following lemma indicates a sufficient condition for minimizing $U_{i+1}-L_{i+1}$.
\begin{lemma}
	\label{lemm:rephrase}
	To minimize $U_{i+1}-L_{i+1}$, it is sufficient to construct $S_{i+1}$ from $S_i$ such that the following two integral quantities are maximized:
	\begin{equation}
		\mathcal{I}(R_{i+1}),\ \mathcal{I}(S_{i+1}).
	\end{equation}
\end{lemma}
The proof of Lemma~\ref{lemm:rephrase} is presented in Appendix~\ref{app:proof_of_rephrase}.
In the following, based on Lemma~\ref{lemm:rephrase}, we propose three search strategies using heuristics.
%
%
We regard the search strategy as a mapping $\pi: S_i\mapsto z$, where $z$ is the point in $S_i^c$ to be searched next, and then consider the following three types of mapping $\pi$:
\begin{equation}
	\label{eq:policy}
	\pi_1(S_i) = \argmin_{z\in S_i^c}\left|z - t\right|,\
	\pi_2(S_i) = \argmax_{z\in S_i^c}f(z),\
	\pi_3(S_i) = \argmax_{z\in E}f(z),
\end{equation}
where $E=\left\{\inf(S_i^c)_{\geq t},\sup(S_i^c)_{\leq t}\right\}$, which are both endpoints of the largest interval containing $t$ in $S_i$.
%
%
%
%
Figure~\ref{fig:search_strategies} illustrates these three search strategies.

These three search strategies are employed based on the following heuristic ideas.
First, it is reasonable to conjecture that $\mathcal{I}(R_{i+1})$ has a good chance to be large at points in $S_i^c$ that are close to $t$ because the same result of the algorithm $\mathcal{A}$ as the observed one is more likely to be obtained in the neighborhood of $t$.
Second, it is reasonable to simply conjecture that $\mathcal{I}(S_{i+1})$ has a good chance to be large at points in $S_i^c$ at which $f$ takes large values (note that, the value of $f$ is the only information we can use because we cannot know the length of intervals newly added to $S_{i+1}$ beforehand).

If we prefer to focus on the maximization of the $\mathcal{I}(R_{i+1})$, the choice is made by the mapping $\pi_1$.
On the other hand, if we prefer to focus on the maximization of the $\mathcal{I}(S_{i+1})$, the choice is made by the mapping $\pi_2$.
Finally, if we prefer to maximize both $\mathcal{I}(R_{i+1})$ and $\mathcal{I}(S_{i+1})$, the choice is made by the mapping $\pi_3$.
\begin{figure}[htbp]
	\centering
	\includegraphics[width=0.5\linewidth]{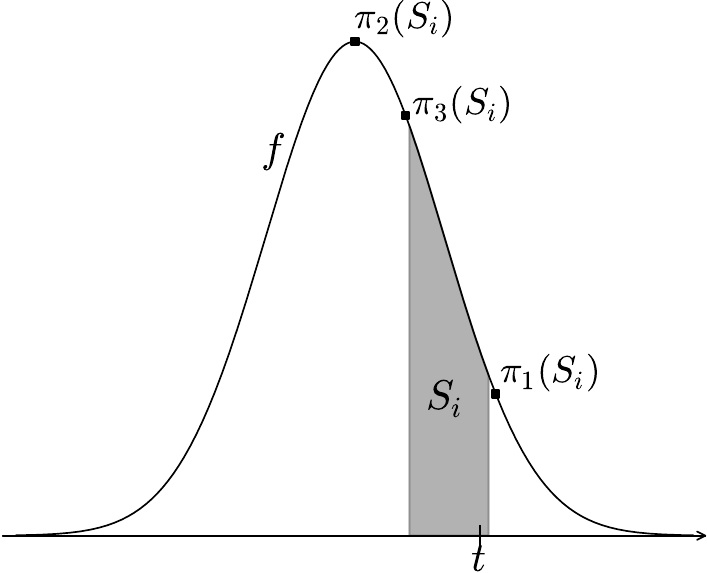}
	\caption{Three types of search strategies}
	\label{fig:search_strategies}
\end{figure}
\subsection{Bounds on Selective Confidence Intervals}
Using a similar approach, it is possible to obtain the upper and lower bounds of the selective confidence interval in the case of selective $z$-test.
See Appendix~\ref{app:extension_to_confidence_interval} for the details.

\newpage
\section{Experiment}
\label{sec:sec4}
We present the experimental results for each of the three types of statistical test: selective $z$-test and $\chi$-test for feature selection in linear models~\citep{sugiyama2021more} and selective $z$-test for attention region identification in deep neural networks (DNNs)~\citep{miwa2023valid}.
We describe the test problems for feature selection in linear models in section~\ref{subsec:sfs} and the test problems for attention region identification in DNNs in section~\ref{subsec:dnn}.
Details of the selection event characterization are deferred to Appendix~\ref{app:si_for_sfs} and \ref{app:si_for_dnn}.
In section~\ref{subsec:synthetic}, we describe the results of synthetic data experiments to confirm the type I error rates, powers and computational costs of the proposed method and several baseline methods.
In section~\ref{subsec:real}, we describe the results of real data experiments to demonstrate the computational advantages of the proposed method.
\subsection{Feature Selection in Linear Models}
\label{subsec:sfs}
As an example of feature selection algorithm for linear models, we consider forward stepwise feature selection (SFS).
In this problem, the data $\bm{D}$ is an $n$-dimensional random \emph{response vector}\footnote{Although the response vector is usually denoted as $\bm{y} \in \mathbb{R}^n$, we use the notation $\bm{D}$, to represent the data vector for different problems in the same notation.}.
We denote $X \in \mathbb{R}^{n \times p}$ by a (non-random) feature matrix, where $p$ is the number of original features.
Let $K$ be the number of the selected features by the SFS method and $\mathcal{M}_K\subset [p]$ be the selected feature set.
\paragraph{Selective $z$-test.}
The true coefficient for each selected feature $j \in \mathcal{M}_K$ is defined as:
\begin{equation}
    \beta_j =
    \left[(X_{\mathcal{M}_K}^\top X_{\mathcal{M}_K})^{-1} X_{\mathcal{M}_K}^\top \bm{D} \right]_j
    =\bm{\eta}^\top\bm{\mu}.
\end{equation}
where
$
    \bm{\eta} =
    X_{\mathcal{M}_K}
    \left(X^\top_{\mathcal{M}_K} X_{\mathcal{M}_K}\right)^{-1} \bm{e}_j
$
with $\bm{e}_j \in \mathbb{R}^{|\mathcal{M}_K|}$ being a basis vector with a $1$ at the $j^\mathrm{th}$ position.
Note that $\bm{\eta}$ depends on $j$, but we omit the dependence for notational simplicity.
For the inference on the $j^\mathrm{th}$ selected feature, we consider the statistical test for each true coefficient $\beta_j = \bm{\eta}^\top \bm{\mu}$ as follows:
\begin{equation}
    \label{eq:hypotheses_sfs}
    \mathrm{H}_{0, j}: \beta_j=0 \quad \text{vs.} \quad
    \mathrm{H}_{1, j}: \beta_j \neq 0.
\end{equation}
\paragraph{Selective $\chi$-test.}
Another possible interest in SFS is to make an inference on a subset of selected features $g\subset \mathcal{M}_K$.
In this case,
it is sufficient to test whether the true mean $\bm{\mu}$ is correlated with a subgroup $X_g$ after regressing out all the other features $\mathcal{M}_k \setminus g$.
Then, we consider the statistical test for
the magnitude of the projection of $\bm{\mu}$ onto the column space of $(I_n - P_{\mathcal{M}_K\setminus g})X_g$ as follows:
\begin{equation}
    \label{eq:hypotheses_sfs_chi}
    \mathrm{H}_{0, g}: \left\|P\bm{\mu}\right\|=0 \quad \text{vs.} \quad
    \mathrm{H}_{1, g}: \left\|P\bm{\mu}\right\| \neq 0,
\end{equation}
where $P$ is the projection matrix onto the column space of $(I_n - P_{\mathcal{M}_K\setminus g})X_g$.
The projection matrix $P$ is obtained by using the singular value decomposition $(I_n - P_{\mathcal{M}_K\setminus g})X_g=U_gS_gV_g^\top$ as $P=U_gU_g^\top$.
\subsection{Attention Region Identification in DNNs}
\label{subsec:dnn}
As an example of a data-driven hypothesis obtained by a DNN, let us consider the problem of quantifying the statistical significance of a saliency map in image classification problems.
In this problem, the data $\bm{D}$ is an $n$-dimensional vector of pixel values in an image.
Consider a trained DNN for an image classification problem.
Given a test image to the trained DNN, a saliency method such as CAM or Grad-CAM visualizes the regions of the test image so-called \emph{salient region} that contribute the most to the classification made by the DNN.
Our goal is to quantify the statistical significance of the saliency region.
Specifically, we consider a subset of pixels in the test image whose saliency map values are greater than a certain threshold and call this subset of pixels as salient region while the rest of the pixels are called non-salient region.
We denote the salient and non-salient regions as $\mathcal{O}_{\bm{D}} \subset [n]$ and $\mathcal{B}_{\bm{D}} = [n] \setminus \mathcal{O}_{\bm{D}}$, respectively (where $\mathcal{O}$ and $\mathcal{B}$ stand for Object and Background, respectively).
\paragraph{Selective $z$-test.}
To test the statistical significance of the saliency region, we consider a statistical test for the following hypotheses:
\begin{equation}
    \label{eq:hypotheses_dnn}
    \mathrm{H}_{0}: \bar{\mu}_{\mathcal{O}_{\bm{D}}}=\bar{\mu}_{\mathcal{B}_{\bm{D}}}
    \quad \text{vs.} \quad
    \mathrm{H}_{1}: \bar{\mu}_{\mathcal{O}_{\bm{D}}} \neq \bar{\mu}_{\mathcal{B}_{\bm{D}}},
\end{equation}
where $\bar{\mu}_{\mathcal{O}_{\bm{D}}} = \frac{1}{|\mathcal{O}_{\bm{D}}|}\sum_{i \in \mathcal{O}_{\bm{D}}} \mu_i$ and $\bar{\mu}_{\mathcal{B}_{\bm{D}}}$ is similarly defined.
This null hypothesis is equivalent to the $\bm{\eta}^\top\bm{\mu}=0$, where
$
    \bm{\eta} =
    \frac{1} {|\mathcal{O}_{\bm{D}}|} \mathbf{1}^n_{\mathcal{O}_{\bm{D}}} -
    \frac{1} {|\mathcal{B}_{\bm{D}}|} \mathbf{1}^n_{\mathcal{B}_{\bm{D}}}
$
with $\mathbf{1}^n_{\mathcal{C}} \in \mathbb{R}^n$ being a vector whose elements belonging to a set $\mathcal{C}$ are set to 1, and 0 otherwise.
\subsection{Synthetic Data Experiments}
\label{subsec:synthetic}
We compared our proposed method with the \texttt{naive} method, over-conditioning (\texttt{OC}) method, and the PP-based SI with exhaustive search (\texttt{exhaustive}).
We first examined the type I error rate, the power and the computation time for the four types of proposed method: \texttt{proposed(precision)} with precision $0.1\%$ and $1.0\%$, and \texttt{proposed(decision)} at significance level $0.01$ and $0.05$, by using the search strategy $\pi_3$.
Furthermore, we conducted a comparison of the number of iterations (search count) of Procedure~\ref{alg:simple} from the following two perspectives.
In the first perspective, we compared the types of the aforementioned four types of proposed method and examined how the search count varies based on differences in precisions and significance levels.
In the second perspective, we compared the three search strategies in section~\ref{subsec:search_strategy} in the cases of \texttt{proposed(precision)} with precision $0.1\%$ and \texttt{proposed(decision)} at significance level $0.05$.
More details (methods for comparison, etc.) are deferred to Appendix~\ref{app:experiment_details}
\paragraph{Selective $z$-test in SFS.}
We generated the dataset $\{(\bm{x}_i, D_i)\}_{i\in[n]}$ by $\bm{x}_i\sim\mathcal{N}(0, I_p)$ and $D_i = \bm{x}_i^\top \bm{\beta} + \epsilon_i$ with $\epsilon_i\sim\mathcal{N}(0, 1)$.
We set $p=10, K=5$ for all experiments.
The coefficients $\bm{\beta}$ was set to a $10$-dimensional zero vector for generating null datasets and set to a 10-dimensional vector whose first five elements are $\Delta$'s and the others are $0$'s to generate datasets for testing power.
We set $n\in \{100, 200, 300, 400\}$ for experiments conducted under the null hypothesis and $n=200, \Delta\in \{0.1, 0.2, 0.3, 0.4\}$ for experiments conducted under the alternative hypothesis.
\paragraph{Selective $\chi$-test in SFS.}
We generated the dataset exactly as in the case of the selective $z$-test.
\paragraph{Selective $z$-test in DNN.}
We generated the image $\{D_i=s_i +\epsilon_i\}_{i\in[n]}$ with $\epsilon_i\sim\mathcal{N}(0, 1)$.
Each $s_i$ was set to $s_i=0,\forall i \in [n]$ for generating null images and set to $s_i=\Delta,\forall i \in \mathcal{S},s_i=0,\forall i \notin [n]\setminus \mathcal{S}$ to generate images for testing power, where $\mathcal{S}$ is the true salient region whose location is randomly determined.
We set $d\in \{8, 16, 32, 64\}$ for experiments conducted under the null hypothesis and $d=16, \Delta\in \{1, 2, 3, 4\}$ for experiments conducted under the alternative hypothesis.
The number of instances was set to $n=d^2$.
We obtained salient regions by thresholding at $\tau=0$ using CAM as the saliency method for all experiments.
\paragraph{Numerical results.}
The results of type I error rates are shown in Figures~\ref{fig:fpr_0.05} and \ref{fig:fpr_0.01}.
Four of the five options \texttt{OC}, \texttt{exhaustive}, \texttt{proposed(precision)} and \texttt{proposed(decision)} successfully controlled the type I error rates under the significance levels, whereas \texttt{naive} could not.
Because \texttt{naive} method failed to control the type I error rate, we no longer considered its power.
The results of powers are shown in Figures~\ref{fig:tpr_0.05} and \ref{fig:tpr_0.01}.
We confirmed that the powers of the proposed method are the same as those of \texttt{exhaustive}, indicating that we succeeded in providing selective $p$-values without exhaustive search (note, however, that \texttt{proposed(precision)} with $1.0\%$ precision showed slightly lower powers because the specified precision is insufficient).
The \texttt{OC} method has the lowest power because it considers several extra conditions, which causes the loss of power.
The results of computation time required for the type I error rate experiments and the power experiments are shown in Figures~\ref{fig:time_null} and \ref{fig:time_alternative}.
All of the four types of proposed method are faster than the \texttt{exhaustive} method in all experimental settings.

Finally, the results of search counts for termination criteria and search strategies under the null hypotheses, and under the alternative hypotheses are shown in Figures~\ref{fig:count_null_criteria}, \ref{fig:count_alternative_criteria}, \ref{fig:count_null_strategy} and \ref{fig:count_alternative_strategy}, respectively.
Figures~\ref{fig:count_null_criteria} and \ref{fig:count_alternative_criteria}, which compare the termination criteria, show that \texttt{proposed(decision)} can significantly reduce search counts.
In particular, under the alternative hypotheses, the larger $\Delta$, the more biased the distribution of $p$-values becomes, resulting in a decrease in search counts.
Figures~\ref{fig:count_null_strategy} and \ref{fig:count_alternative_strategy}, which compare the search strategies, show that $\pi_3$ performs best for \texttt{proposed(precision)}.
It is difficult to evaluate the advantage and disadvantage of the three search strategies for \texttt{proposed(decision)}.
\subsection{Real Data Experiments}
\label{subsec:real}
We compared the two types of proposed method (\texttt{proposed(precision)} with precision $0.1\%$ and \texttt{proposed(decision)} at significance level $0.05$) and conventional method (\texttt{exhaustive}) on real datasets.
We used search strategy $\pi_3$ in proposed method.
For two types of tests in SFS, we used the Concrete Compressive Strength dataset in UCI machine learning repository~\citep{concrete_compressive_strength}, which consists of 1030 instances, each consisting of an $8$-dimensional feature vector and one response variable.
We randomly generated 250 sub-sampled datasets with sizes 400 from this dataset and applied the SFS with $K=3$.
For a test in DNN, we used the brain image dataset extracted from the dataset used in~\citep{buda2019association}, which includes 939 and 941 images with and without tumors, respectively.
We randomly extracted 250 images each with and without tumors and obtained salient regions by thresholding at $\tau=100$.
The results of search count for the selective $z$-test and $\chi$-test in SFS are shown in Figure~\ref{fig:sfs_real}, and the results of search count for the selective $z$-test in DNN for images with and without tumors are shown in Figure~\ref{fig:dnn_real}.
In all of the results, the proposed method significantly reduce search counts.
\begin{figure}[htbp]
    \centering
    \includegraphics[width=0.7\linewidth]{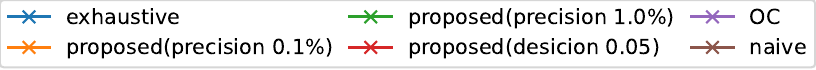}
    \begin{minipage}[b]{0.32\linewidth}
        \centering
        \includegraphics[width=1.0\linewidth]{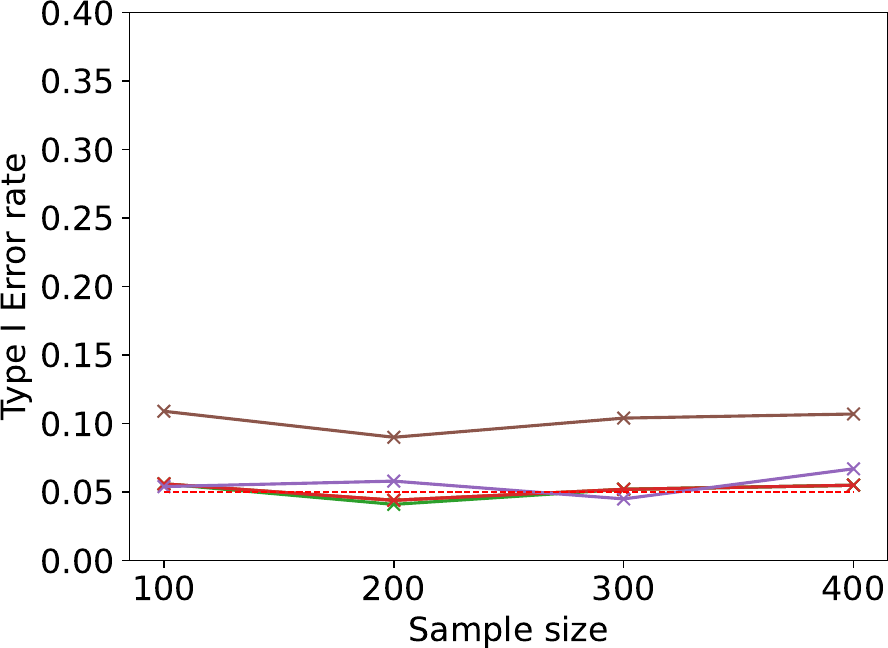}
        \subcaption{SFS ($z$-test)}
    \end{minipage}
    \begin{minipage}[b]{0.32\linewidth}
        \centering
        \includegraphics[width=1.0\linewidth]{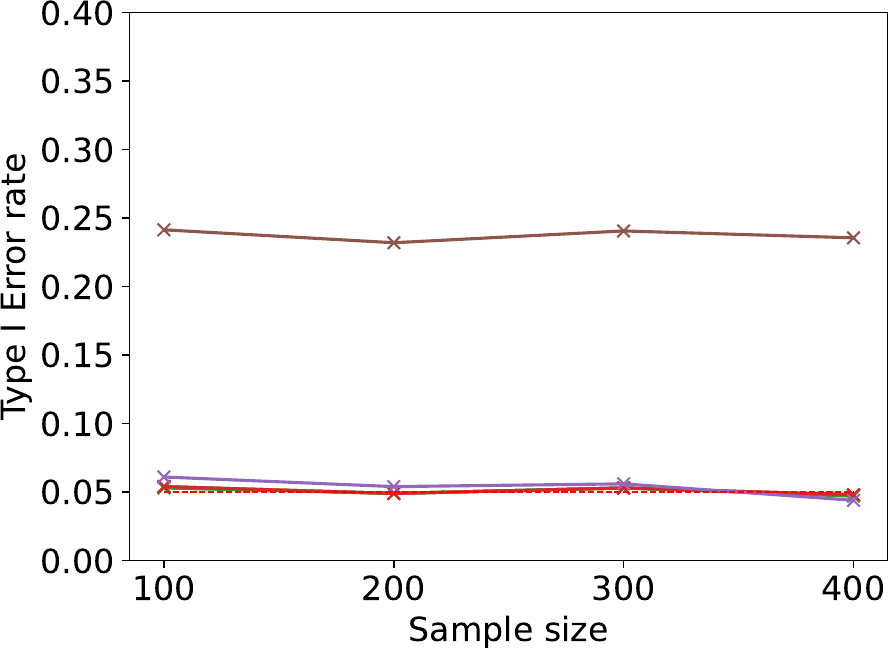}
        \subcaption{SFS ($\chi$-test)}
    \end{minipage}
    \begin{minipage}[b]{0.32\linewidth}
        \centering
        \includegraphics[width=1.0\linewidth]{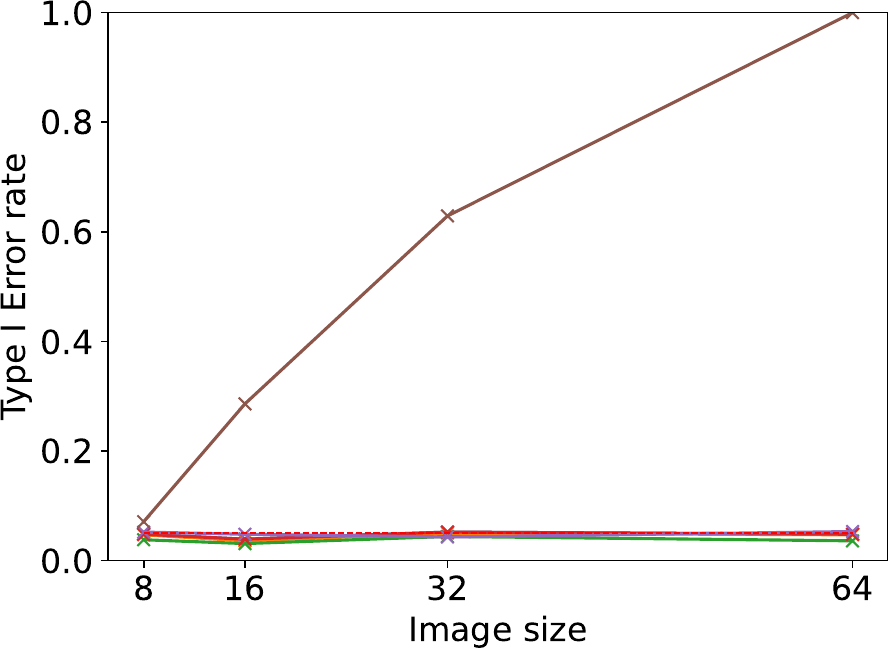}
        \subcaption{DNN ($z$-test)}
    \end{minipage}
    \caption{Type I Error rates (significance level is $0.05$)}
    \label{fig:fpr_0.05}
\end{figure}
\begin{figure}[htbp]
    \centering
    \includegraphics[width=0.7\linewidth]{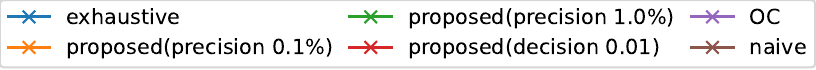}
    \begin{minipage}[b]{0.32\linewidth}
        \centering
        \includegraphics[width=1.0\linewidth]{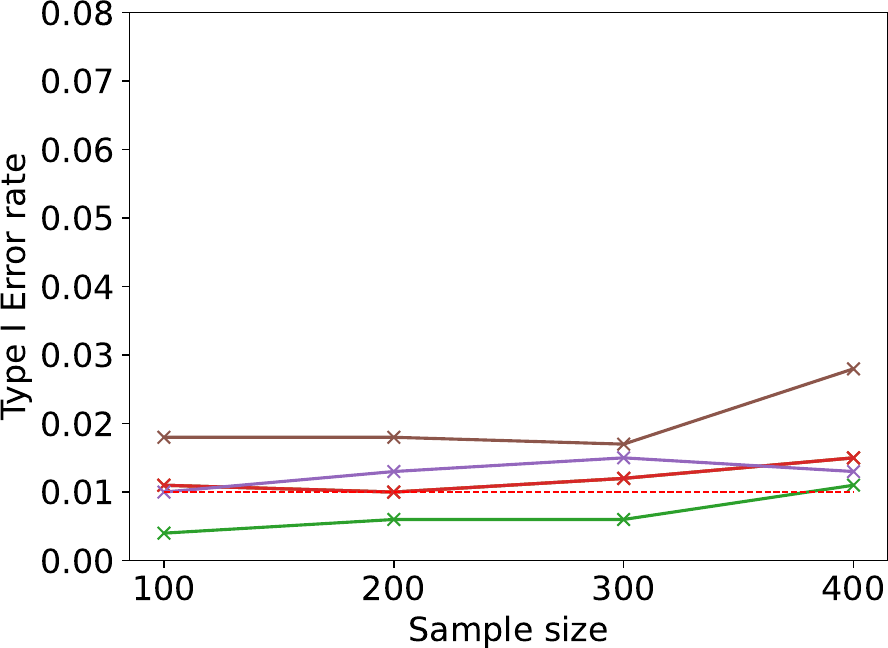}
        \subcaption{SFS ($z$-test)}
    \end{minipage}
    \begin{minipage}[b]{0.32\linewidth}
        \centering
        \includegraphics[width=1.0\linewidth]{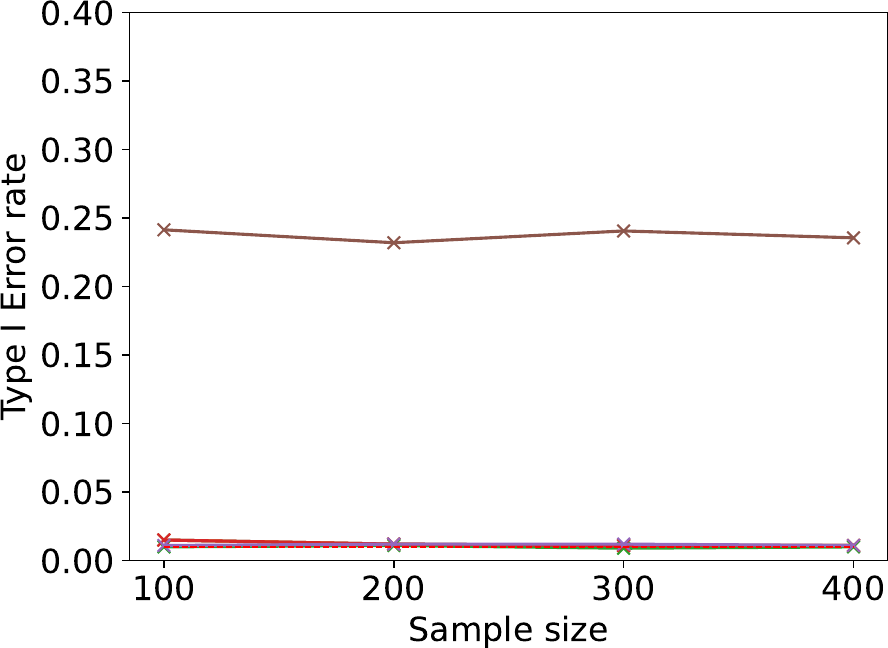}
        \subcaption{SFS ($\chi$-test)}
    \end{minipage}
    \begin{minipage}[b]{0.32\linewidth}
        \centering
        \includegraphics[width=1.0\linewidth]{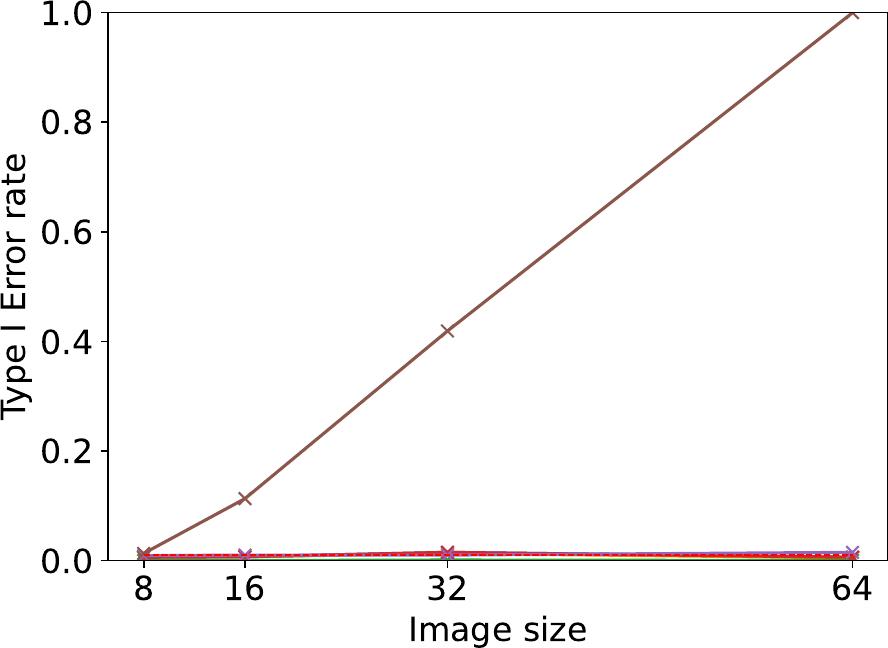}
        \subcaption{DNN ($z$-test)}
    \end{minipage}
    \caption{Type I Error rates (significance level is $0.01$)}
    \label{fig:fpr_0.01}
\end{figure}
\begin{figure}[htbp]
    \centering
    \includegraphics[width=0.7\linewidth]{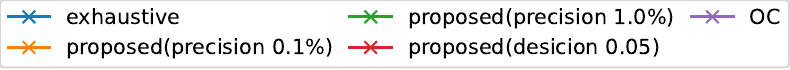}
    \begin{minipage}[b]{0.32\linewidth}
        \centering
        \includegraphics[width=1.0\linewidth]{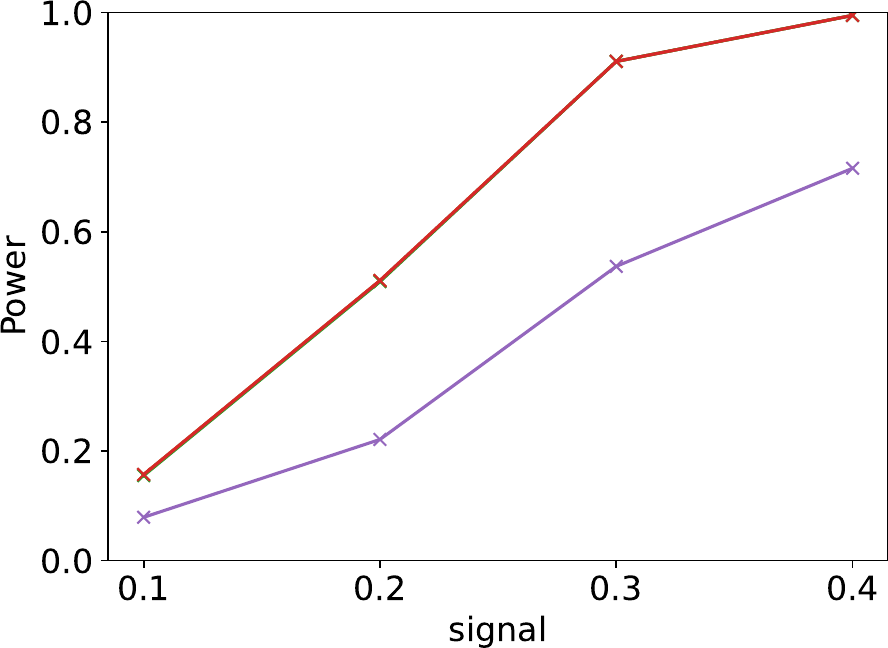}
        \subcaption{SFS ($z$-test)}
    \end{minipage}
    \begin{minipage}[b]{0.32\linewidth}
        \centering
        \includegraphics[width=1.0\linewidth]{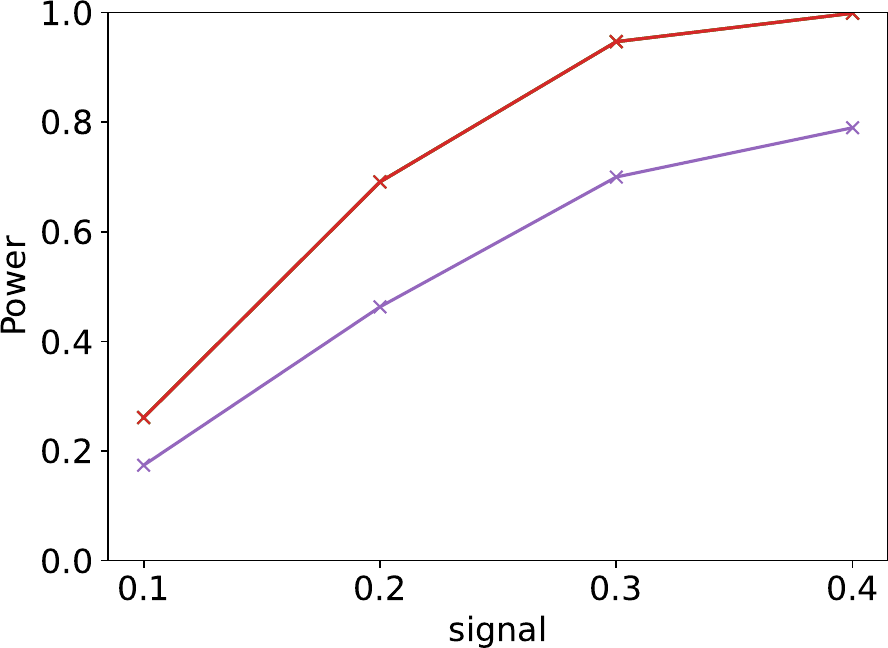}
        \subcaption{SFS ($\chi$-test)}
    \end{minipage}
    \begin{minipage}[b]{0.32\linewidth}
        \centering
        \includegraphics[width=1.0\linewidth]{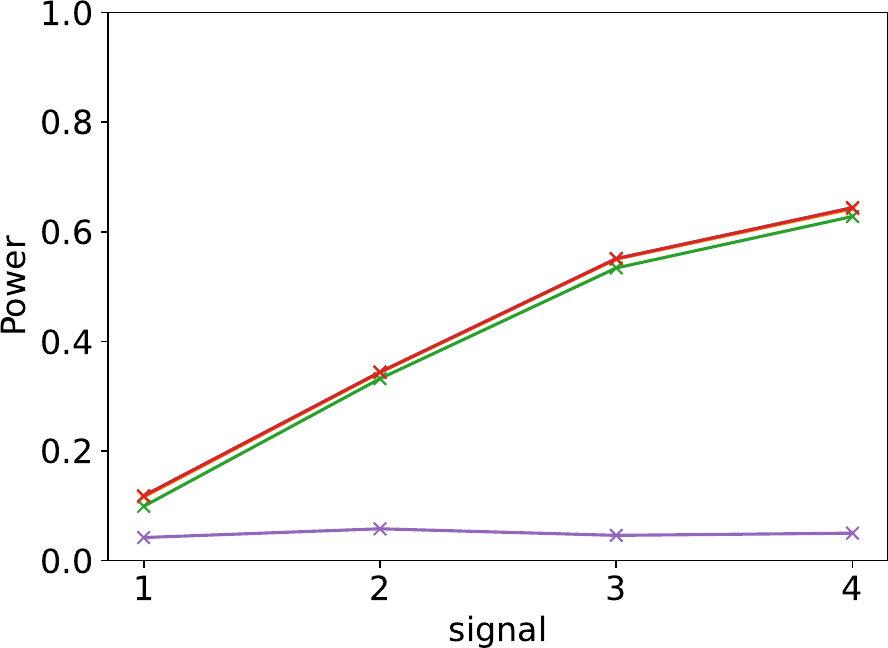}
        \subcaption{DNN ($z$-test)}
    \end{minipage}
    \caption{Powers (significance level is $0.05$)}
    \label{fig:tpr_0.05}
\end{figure}
\begin{figure}[htbp]
    \centering
    \includegraphics[width=0.7\linewidth]{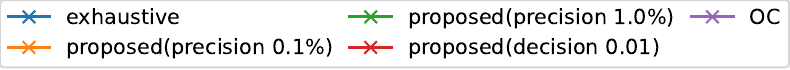}
    \begin{minipage}[b]{0.32\linewidth}
        \centering
        \includegraphics[width=1.0\linewidth]{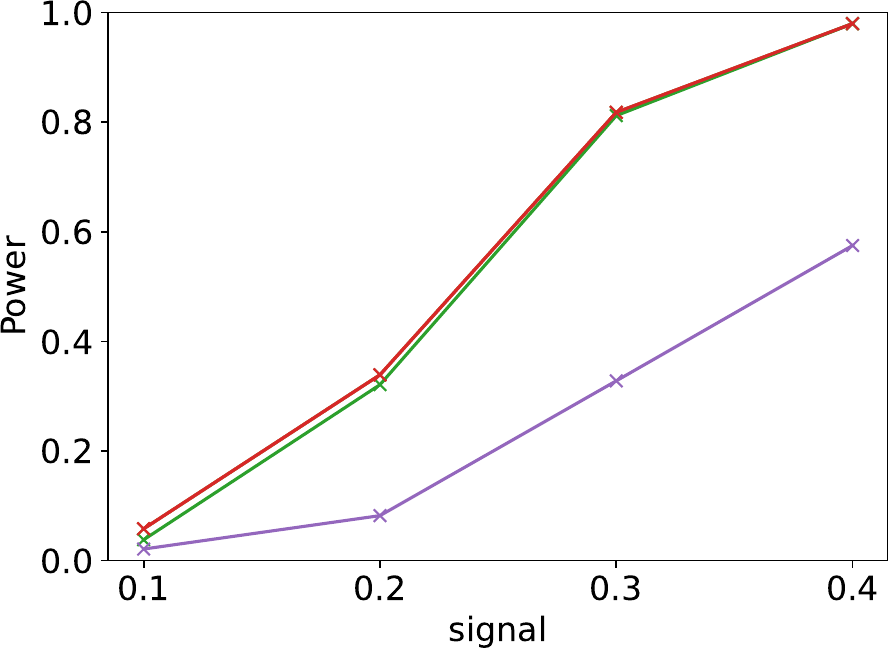}
        \subcaption{SFS ($z$-test)}
    \end{minipage}
    \begin{minipage}[b]{0.32\linewidth}
        \centering
        \includegraphics[width=1.0\linewidth]{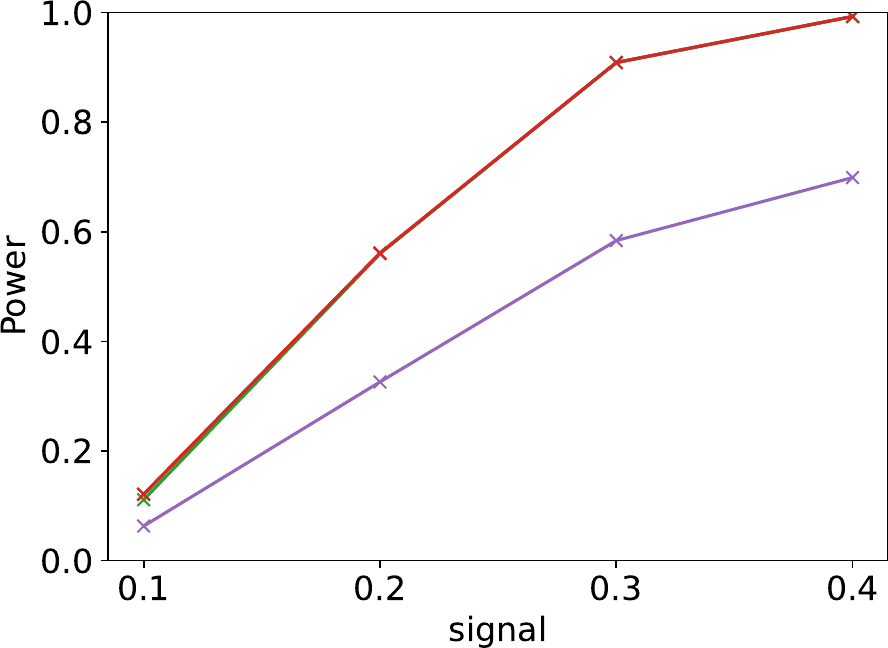}
        \subcaption{SFS ($\chi$-test)}
    \end{minipage}
    \begin{minipage}[b]{0.32\linewidth}
        \centering
        \includegraphics[width=1.0\linewidth]{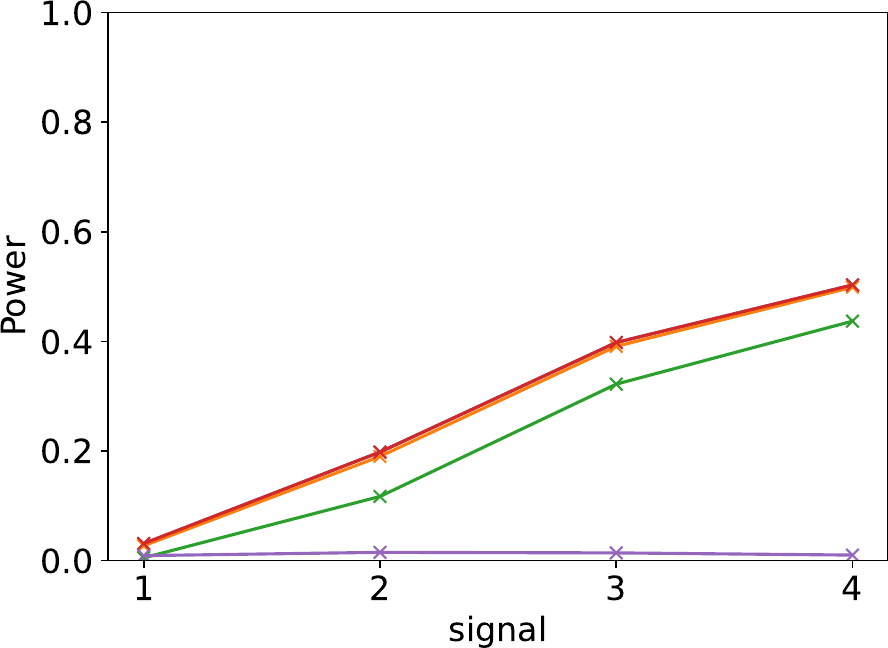}
        \subcaption{DNN ($z$-test)}
    \end{minipage}
    \caption{Powers (significance level is $0.01$)}
    \label{fig:tpr_0.01}
\end{figure}
\begin{figure}[htbp]
    \centering
    \includegraphics[width=0.8\linewidth]{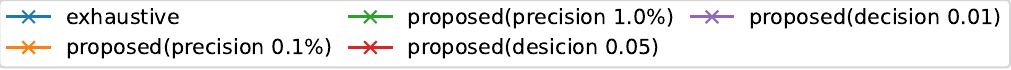}
    \begin{minipage}[b]{0.32\linewidth}
        \centering
        \includegraphics[width=1.0\linewidth]{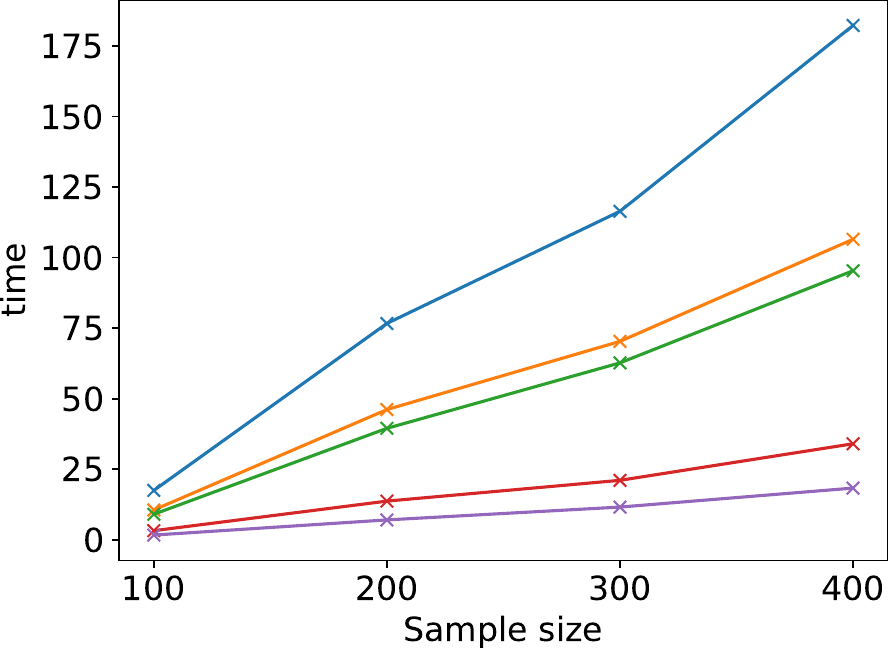}
        \subcaption{SFS ($z$-test)}
    \end{minipage}
    \begin{minipage}[b]{0.32\linewidth}
        \centering
        \includegraphics[width=1.0\linewidth]{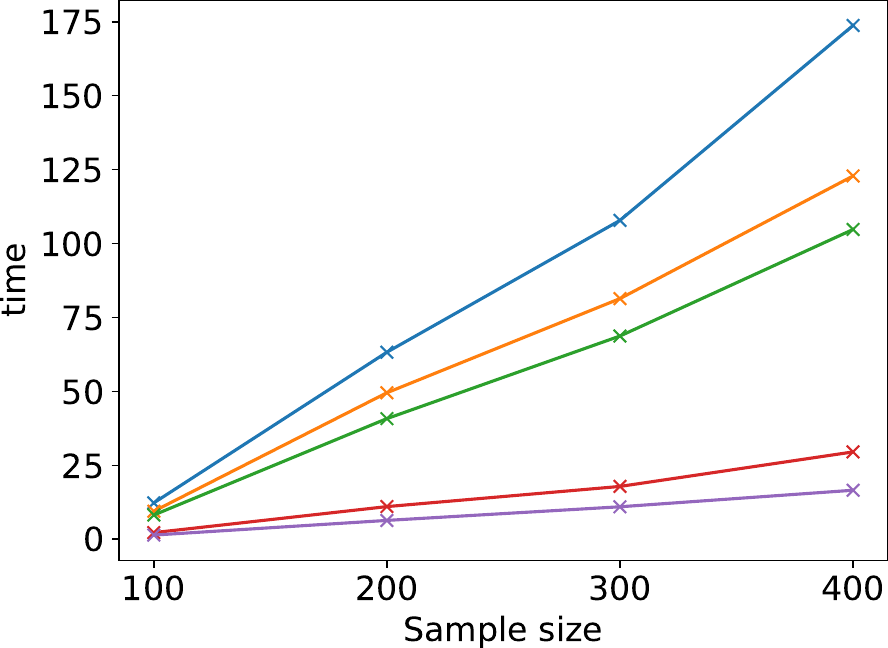}
        \subcaption{SFS ($\chi$-test)}
    \end{minipage}
    \begin{minipage}[b]{0.32\linewidth}
        \centering
        \includegraphics[width=1.0\linewidth]{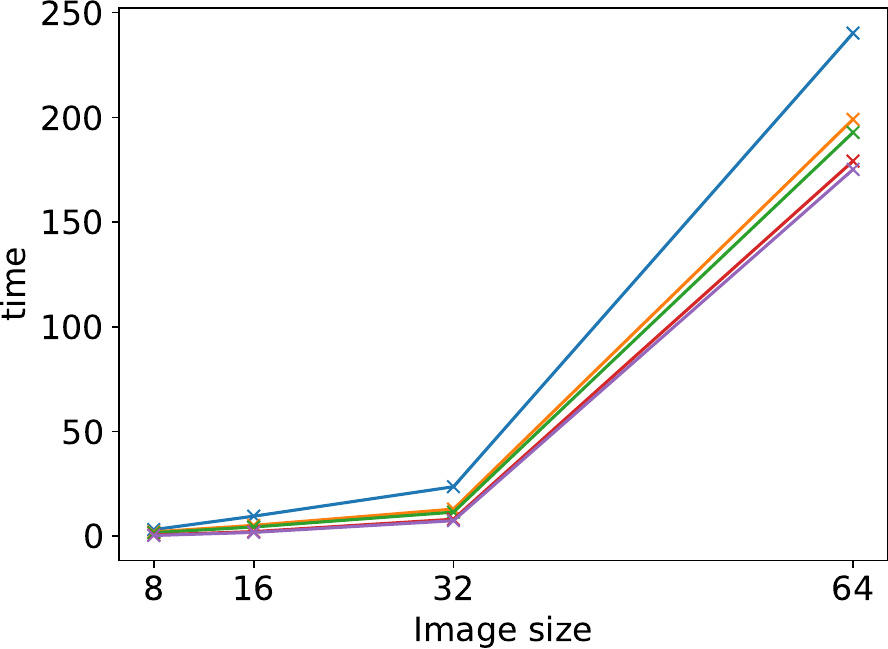}
        \subcaption{DNN ($z$-test)}
    \end{minipage}
    \caption{Computation time under null hypotheses}
    \label{fig:time_null}
\end{figure}
\begin{figure}[htbp]
    \centering
    \includegraphics[width=0.8\linewidth]{Fig/legends/time_and_count.pdf}
    \begin{minipage}[b]{0.32\linewidth}
        \centering
        \includegraphics[width=1.0\linewidth]{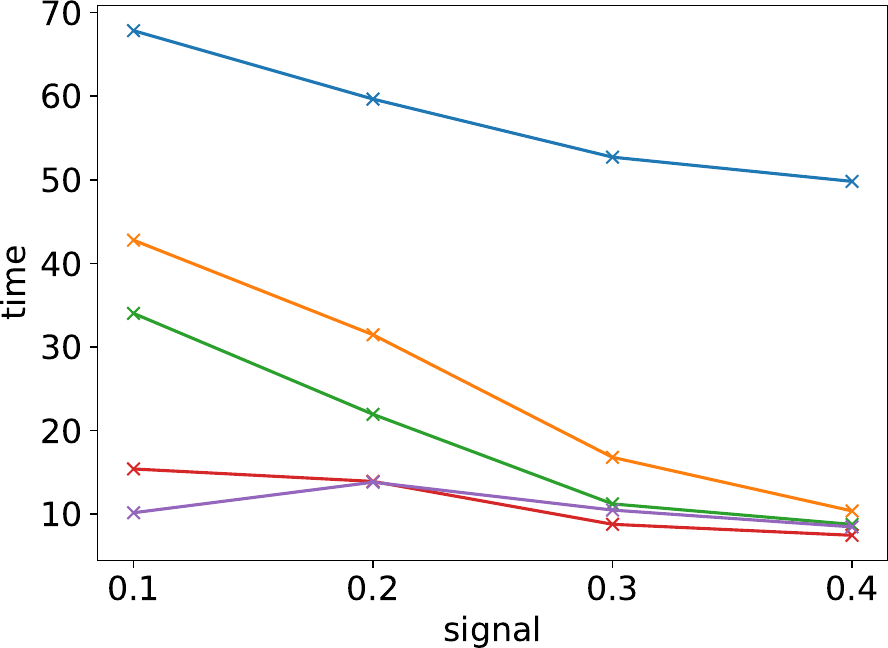}
        \subcaption{SFS ($z$-test)}
    \end{minipage}
    \begin{minipage}[b]{0.32\linewidth}
        \centering
        \includegraphics[width=1.0\linewidth]{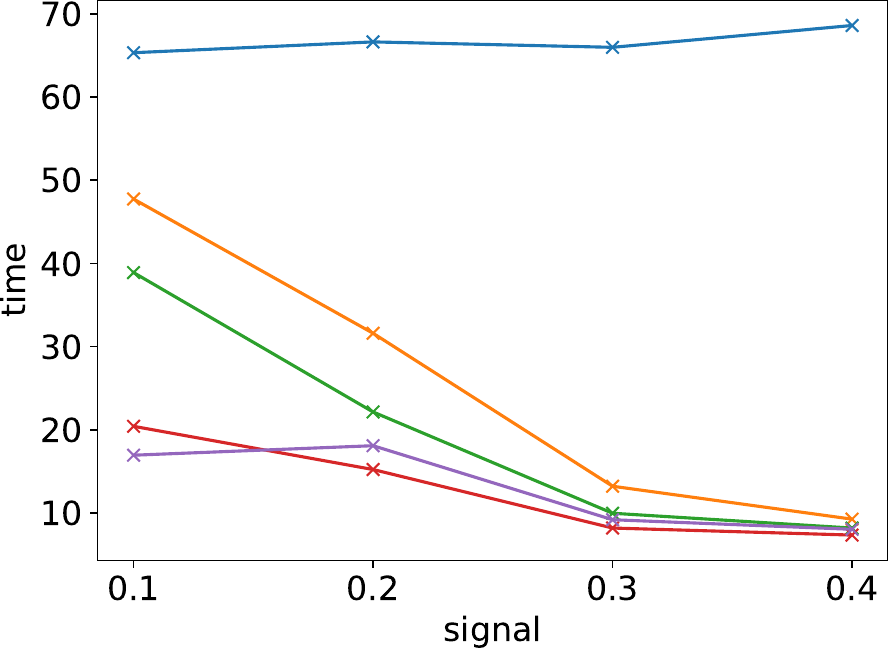}
        \subcaption{SFS ($\chi$-test)}
    \end{minipage}
    \begin{minipage}[b]{0.32\linewidth}
        \centering
        \includegraphics[width=1.0\linewidth]{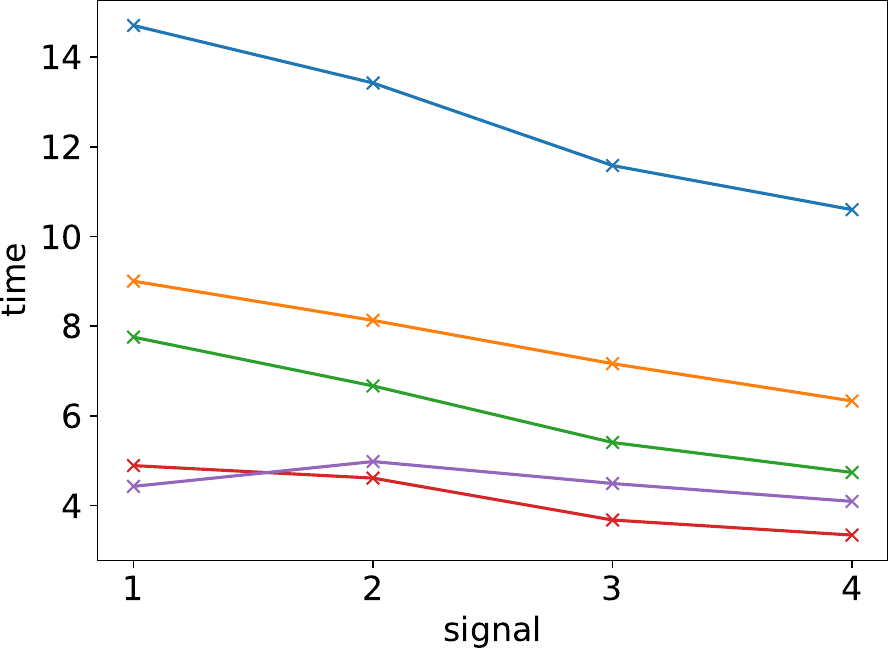}
        \subcaption{DNN ($z$-test)}
    \end{minipage}
    \caption{Computation time under alternative hypotheses}
    \label{fig:time_alternative}
\end{figure}
\begin{figure}[htbp]
    \centering
    \includegraphics[width=0.8\linewidth]{Fig/legends/time_and_count.pdf}
    \begin{minipage}[b]{0.32\linewidth}
        \centering
        \includegraphics[width=1.0\linewidth]{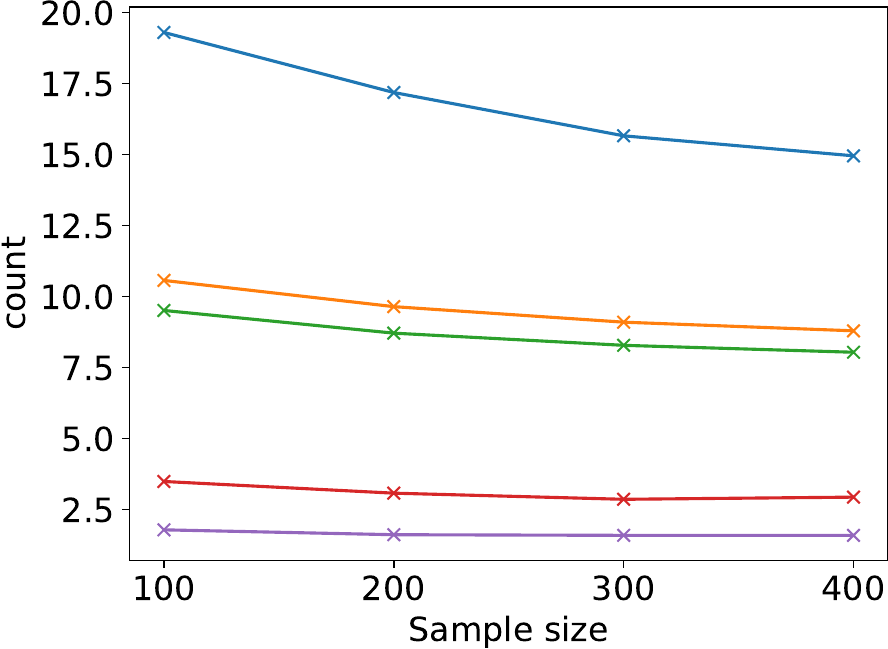}
        \subcaption{SFS ($z$-test)}
    \end{minipage}
    \begin{minipage}[b]{0.32\linewidth}
        \centering
        \includegraphics[width=1.0\linewidth]{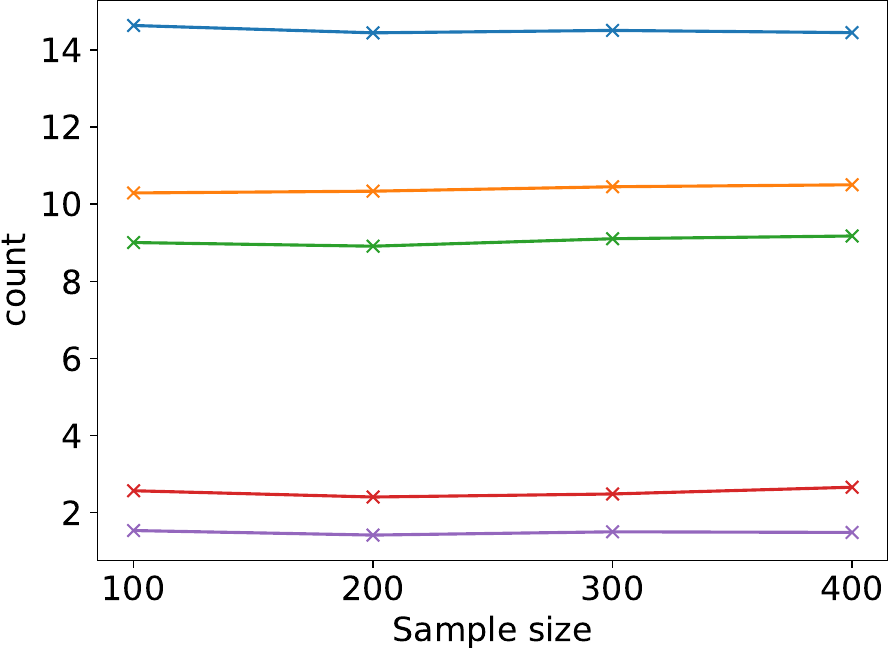}
        \subcaption{SFS ($\chi$-test)}
    \end{minipage}
    \begin{minipage}[b]{0.32\linewidth}
        \centering
        \includegraphics[width=1.0\linewidth]{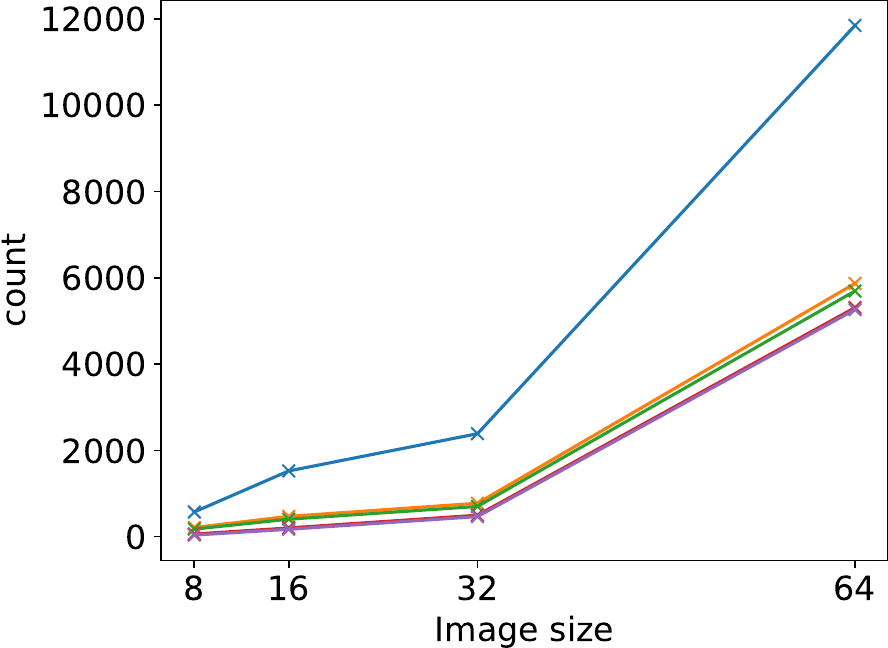}
        \subcaption{DNN ($z$-test)}
    \end{minipage}
    \caption{Search counts for termination criteria under null hypotheses}
    \label{fig:count_null_criteria}
\end{figure}
\begin{figure}[htbp]
    \centering
    \includegraphics[width=0.8\linewidth]{Fig/legends/time_and_count.pdf}
    \begin{minipage}[b]{0.32\linewidth}
        \centering
        \includegraphics[width=1.0\linewidth]{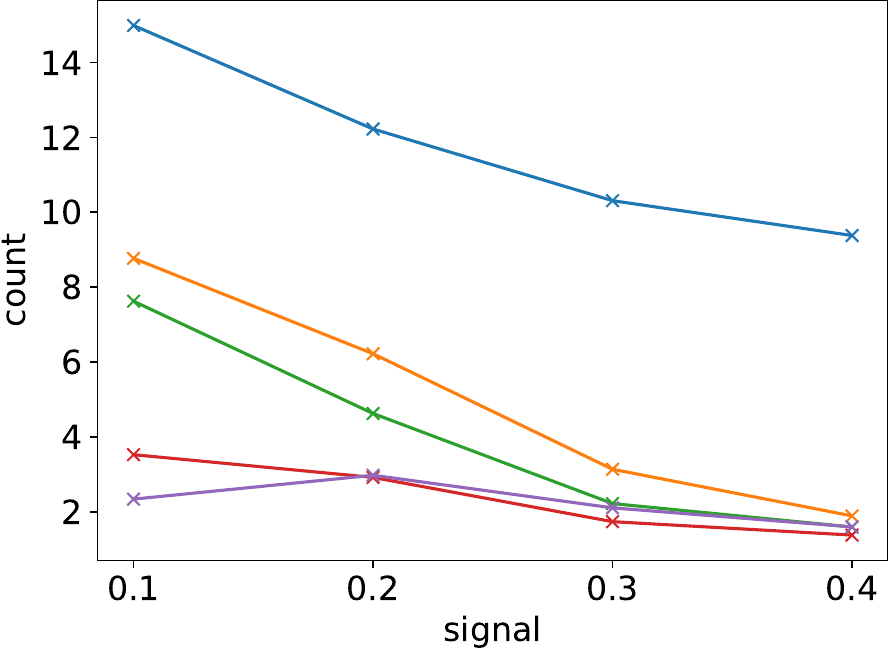}
        \subcaption{SFS ($z$-test)}
    \end{minipage}
    \begin{minipage}[b]{0.32\linewidth}
        \centering
        \includegraphics[width=1.0\linewidth]{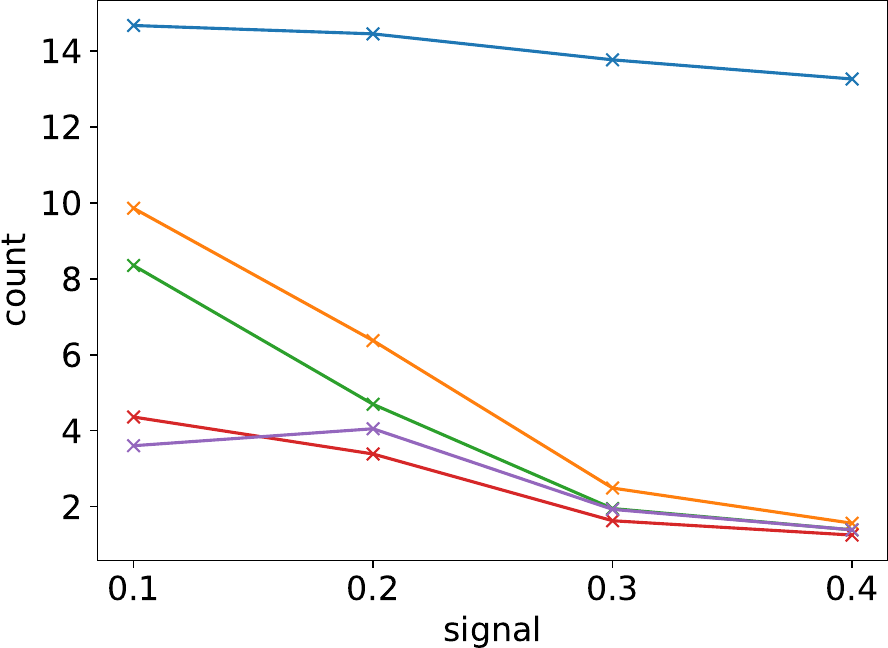}
        \subcaption{SFS ($\chi$-test)}
    \end{minipage}
    \begin{minipage}[b]{0.32\linewidth}
        \centering
        \includegraphics[width=1.0\linewidth]{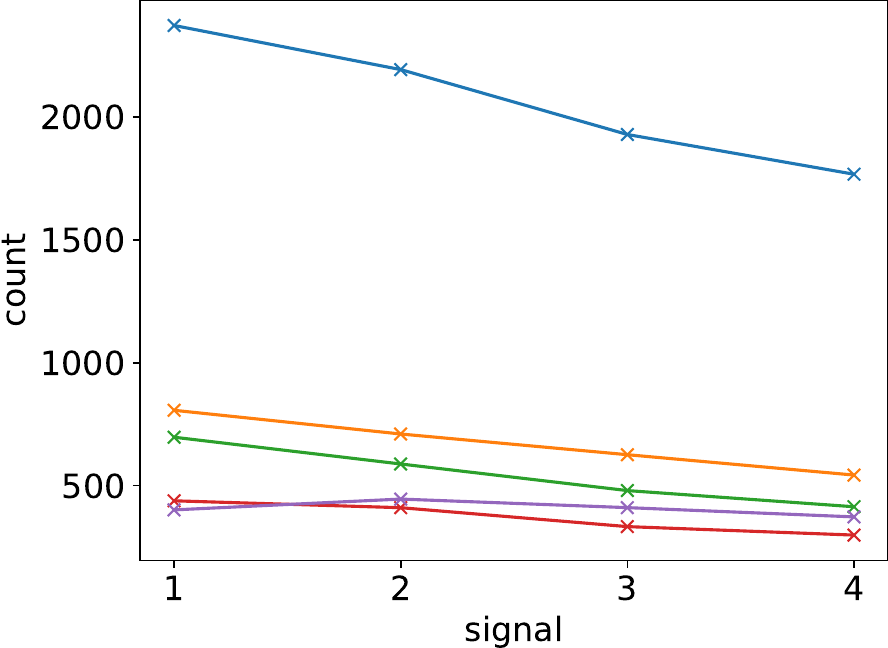}
        \subcaption{DNN ($z$-test)}
    \end{minipage}
    \caption{Search counts for termination criteria under alternative hypotheses}
    \label{fig:count_alternative_criteria}
\end{figure}
\begin{figure}[htbp]
    \centering
    \includegraphics[width=0.9\linewidth]{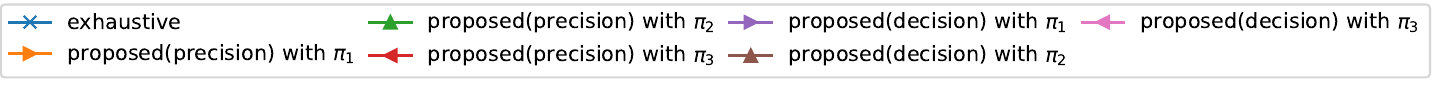}
    \begin{minipage}[b]{0.32\linewidth}
        \centering
        \includegraphics[width=1.0\linewidth]{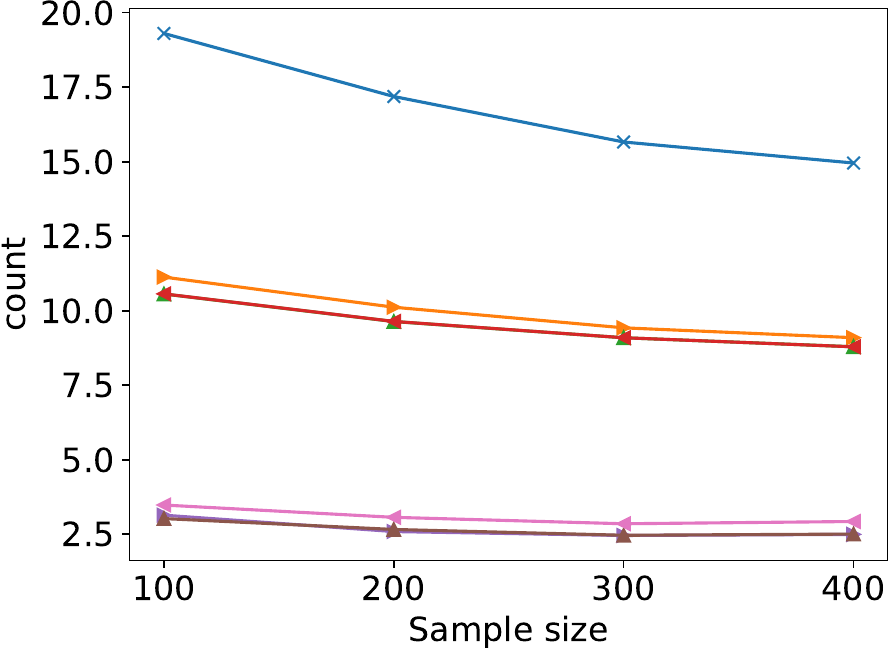}
        \subcaption{SFS ($z$-test)}
    \end{minipage}
    \begin{minipage}[b]{0.32\linewidth}
        \centering
        \includegraphics[width=1.0\linewidth]{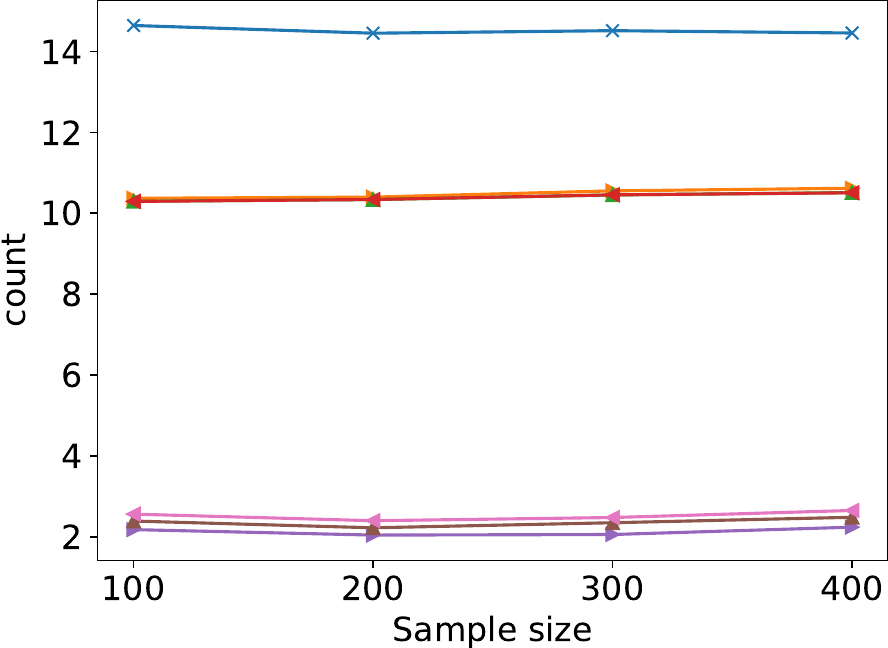}
        \subcaption{SFS ($\chi$-test)}
    \end{minipage}
    \begin{minipage}[b]{0.32\linewidth}
        \centering
        \includegraphics[width=1.0\linewidth]{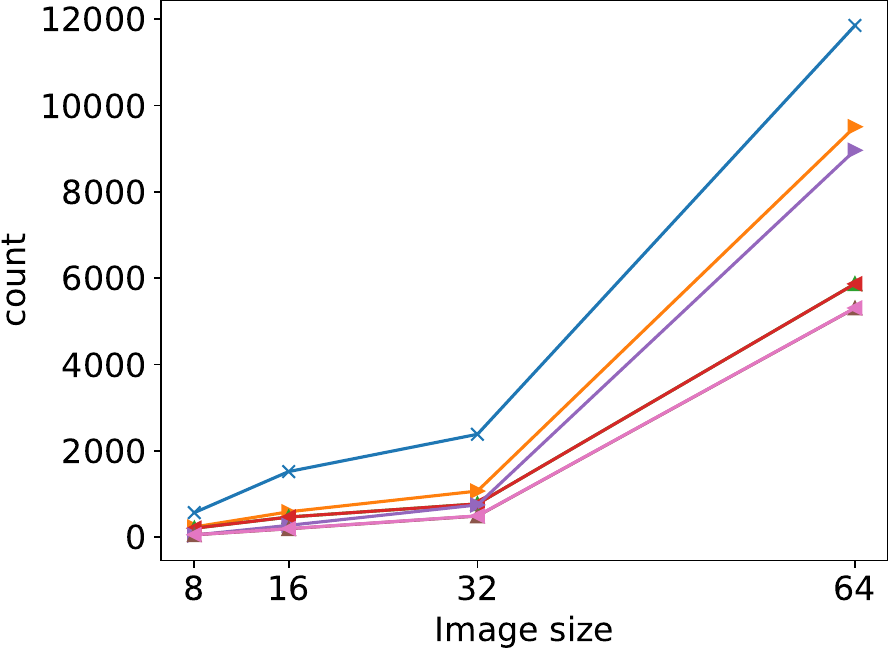}
        \subcaption{DNN ($z$-test)}
    \end{minipage}
    \caption{Search counts for search strategies under null hypotheses}
    \label{fig:count_null_strategy}
\end{figure}
\begin{figure}[htbp]
    \centering
    \includegraphics[width=0.9\linewidth]{Fig/legends/strategy.pdf}
    \begin{minipage}[b]{0.32\linewidth}
        \centering
        \includegraphics[width=1.0\linewidth]{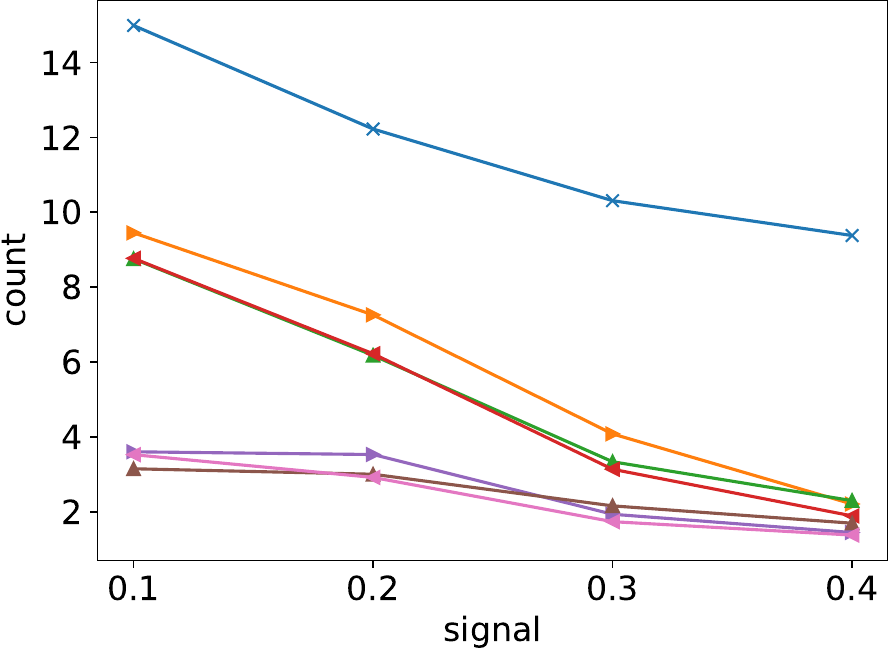}
        \subcaption{SFS ($z$-test)}
    \end{minipage}
    \begin{minipage}[b]{0.32\linewidth}
        \centering
        \includegraphics[width=1.0\linewidth]{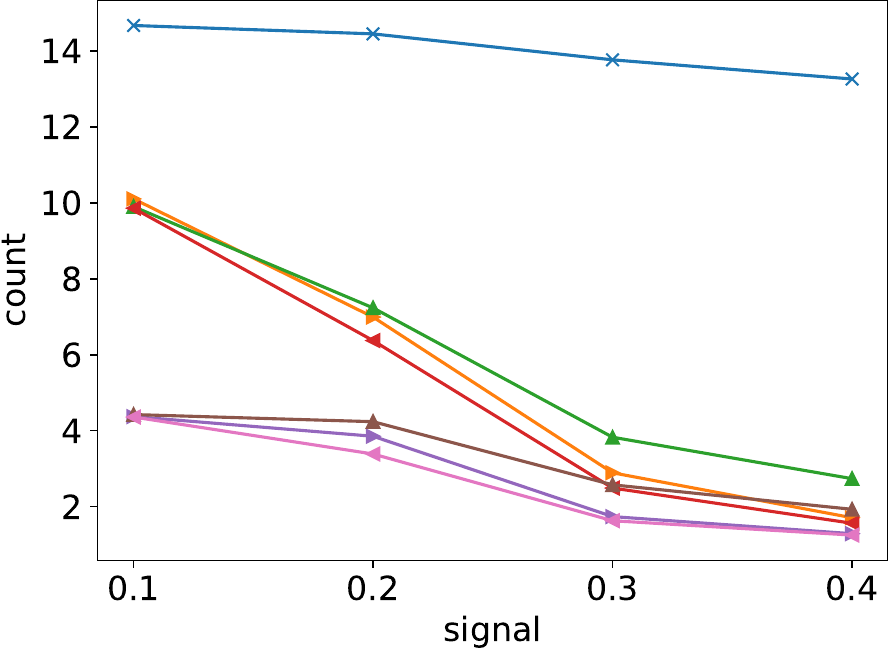}
        \subcaption{SFS ($\chi$-test)}
    \end{minipage}
    \begin{minipage}[b]{0.32\linewidth}
        \centering
        \includegraphics[width=1.0\linewidth]{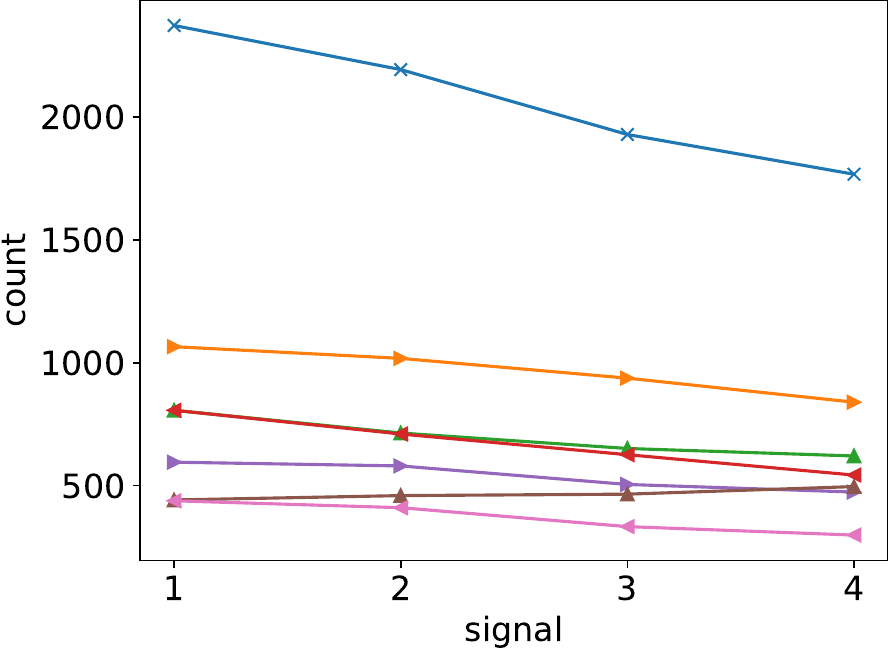}
        \subcaption{DNN ($z$-test)}
    \end{minipage}
    \caption{Search counts for search strategies under alternative hypotheses}
    \label{fig:count_alternative_strategy}
\end{figure}
\begin{figure}[htbp]
    \begin{minipage}[b]{0.49\linewidth}
        \centering
        \includegraphics[width=1.0\linewidth]{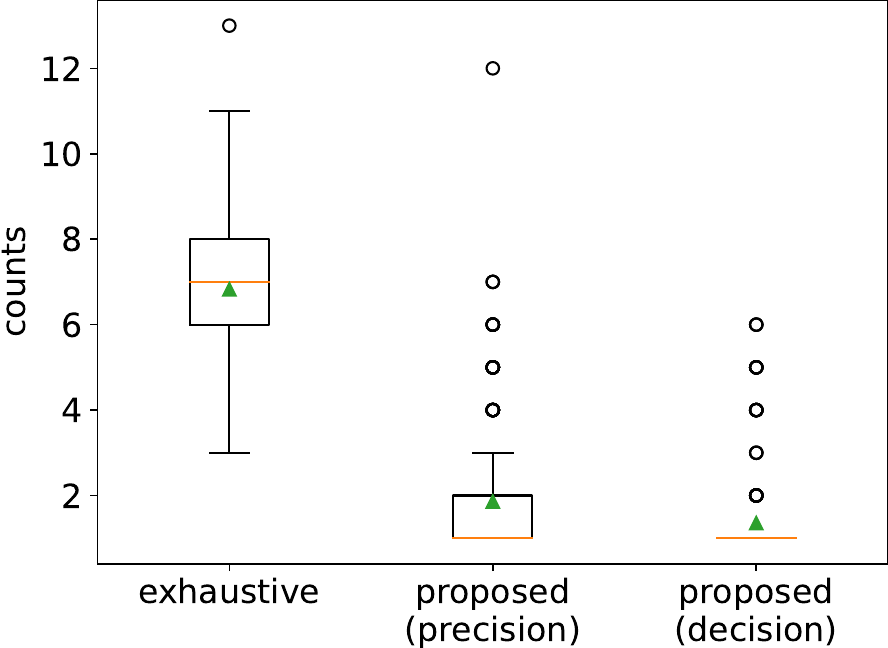}
        \subcaption{$z$-test}
    \end{minipage}
    \begin{minipage}[b]{0.49\linewidth}
        \centering
        \includegraphics[width=1.0\linewidth]{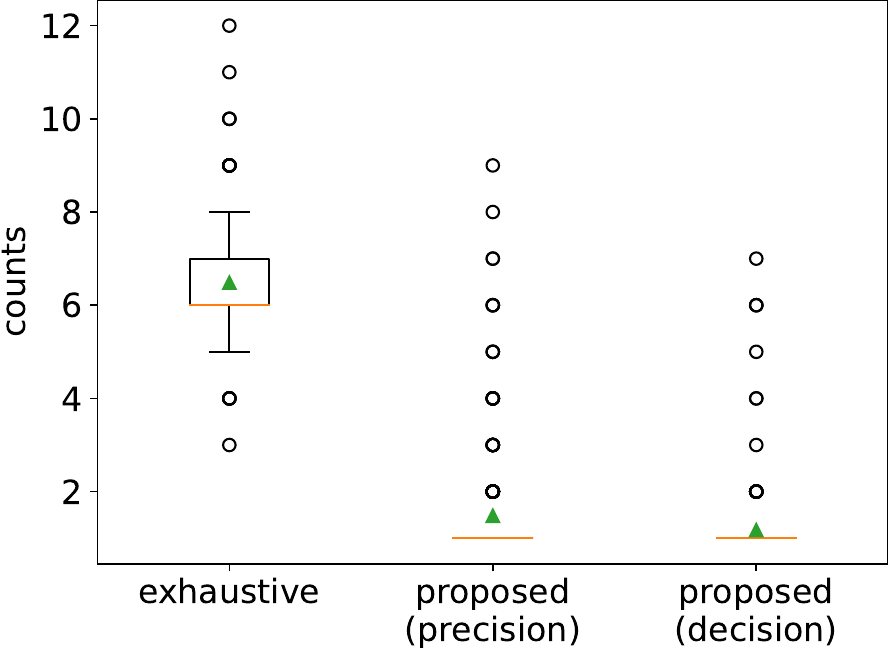}
        \subcaption{$\chi$-test}
    \end{minipage}
    \caption{Search counts for SFS on real dataset}
    \label{fig:sfs_real}
\end{figure}
\begin{figure}[htbp]
    \begin{minipage}[b]{0.49\linewidth}
        \centering
        \includegraphics[width=1.0\linewidth]{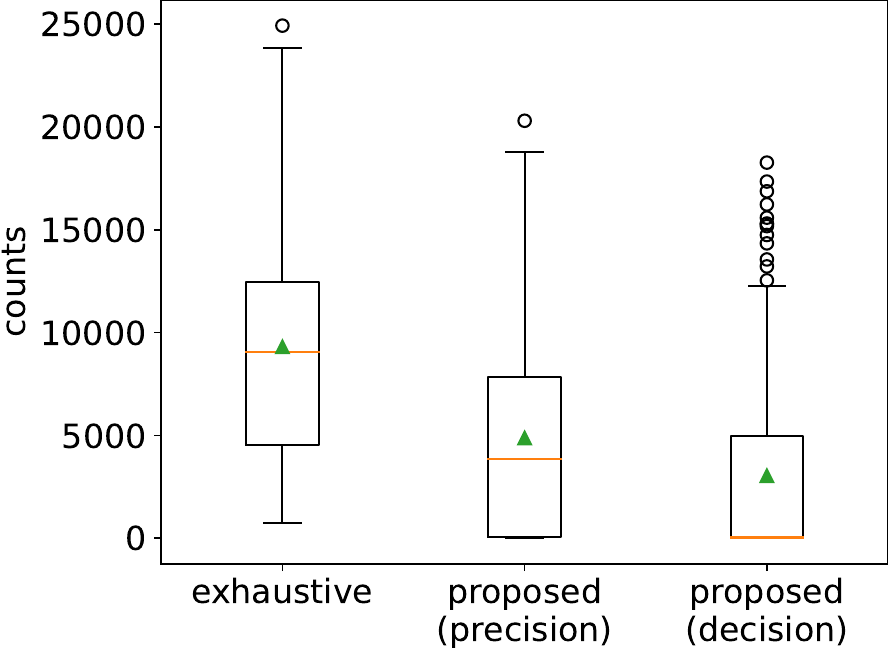}
        \subcaption{with tumors}
    \end{minipage}
    \begin{minipage}[b]{0.49\linewidth}
        \centering
        \includegraphics[width=1.0\linewidth]{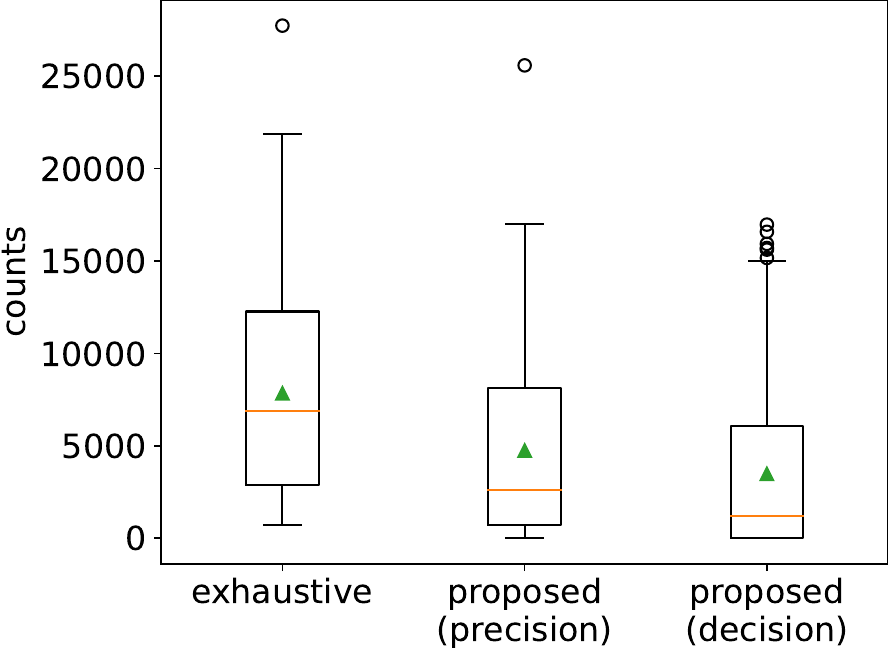}
        \subcaption{without tumors}
    \end{minipage}
    \caption{Search counts for DNN on real dataset}
    \label{fig:dnn_real}
\end{figure}

\newpage
\section{Conclusion}
\label{sec:sec5}
In this paper, we propose a method for computing the lower and upper bounds of selective $p$-values within the context of PP-based SI, without the need for exhaustive search of the data space.
These bounds allow us to appropriately terminate the $p$-value computation procedure based on desired precisions or significance levels.
In the conventional PP-based SI method, exhaustive exploration of the data space was necessary, resulting in significant computational costs for complex ML algorithms such as deep learning models.
Our proposed method can significantly reduce the computational cost while ensuring the computation of selective $p$-values with theoretically guaranteed precision.
The proposed method can be effectively applied to various SI tasks for various data-driven hypotheses.
Future works to further reducing computational costs will involve exploring search strategies that leverage parallelizable computing resources.
Furthermore, in sampling-based SI methods~\citep{terada2017selective,panigrahi2022approximate}, it may be possible to develop an approach that obtains bounds for selective $p$-values without requiring exhaustive sampling, based on similar ideas to our proposed method.


\newpage
\subsection*{Acknowledgement}
This work was partially supported by MEXT KAKENHI (20H00601), JST CREST (JPMJCR21D3, JPMJCR22N2), JST Moonshot R\&D (JPMJMS2033-05), JST AIP Acceleration Research (JPMJCR21U2), NEDO (JPNP18002, JPNP20006) and RIKEN Center for Advanced Intelligence Project.

\clearpage

\appendix

\newpage
\section{Proofs of Theorems and Lemmas}
\label{app:}
\subsection{Proof of Theorem~\ref{thm:main}}
\label{app:proof_of_main_theorem}
%
%
From the construction in Procedure~\ref{alg:simple}, both of $\{S_i\}_{i \in \mathbb{N}}$ and $\{R_i\}_{i \in \mathbb{N}}$ are obviously monotonically increasing sequences.
Due to the property of the algorithm $\mathcal{A}$, the search over all real numbers is completed in at most additive times, and $S_\infty=\mathbb{R}$.
Thus, the $\mathcal{Z}$ is equal to $\mathbb{R}_\infty$.
In addition, since the result obtained from any point in $S_i \setminus R_i$ by $\mathcal{A}$ is guaranteed to be different from the observed one, then we obtain $S_i\cap R_i^c\subseteq R_\infty^c$, and hence $R_\infty\subseteq R_i\cup S_i^c$ (considering the complements).
%
%
%
We define the set family sequence $\left\{\mathcal{R}_i\right\}_{i\in\mathbb{N}}$ using the subset $\mathcal{R}_i$ of $\mathcal{B}(\mathbb{R})$ shown in
\begin{equation}
	\mathcal{R}_i =
	\left\{
	R\in\mathcal{B}(\mathbb{R})\mid R_i\subseteq R \subseteq R_i\cup S_i^c
	\right\}.
\end{equation}
From the properties of $\{S_i\}_{i \in \mathbb{N}}$ and $\{R_i\}_{i \in \mathbb{N}}$, for any $i\in\mathbb{N}$, $R_\infty\in\mathcal{R}_i$ holds and hence the inequality
\begin{equation}
	\inf_{R\in\mathcal{R}_i}
	\frac
	{\mathcal{I}(R\setminus [-|t|,|t|])}
	{\mathcal{I}(R)}
	\leq
	p^\mathrm{selective}
	\leq
	\sup_{R\in\mathcal{R}_i}
	\frac
	{\mathcal{I}(R\setminus [-|t|,|t|])}
	{\mathcal{I}(R)}
\end{equation}
holds.
Thus, we only need to show the following two equations:
\begin{align}
	\inf_{R\in\mathcal{R}_i}
	\frac
	{\mathcal{I}(R\setminus [-|t|,|t|])}
	{\mathcal{I}(R)}
	 & =
	\frac
	{\mathcal{I}(R_i^\mathrm{inf} \setminus [-|t|,|t|])}
	{\mathcal{I}(R_i^\mathrm{inf})},
	\\
	\sup_{R\in\mathcal{R}_i}
	\frac
	{\mathcal{I}(R\setminus [-|t|,|t|])}
	{\mathcal{I}(R)}
	 & =
	\frac
	{\mathcal{I}(R_i^\mathrm{sup}\setminus [-|t|,|t|])}
	{\mathcal{I}(R_i^\mathrm{sup})},
\end{align}
where $R_i^\mathrm{inf}=R_i\cup (S_i^c\cap [-|t|,|t|])$ and $R_i^\mathrm{sup}=R_i\cup (S_i^c\setminus [-|t|,|t|])$ are elements of $\mathcal{R}_i$, for any $i\in\mathbb{N}$.
Now, the term to be optimized can be rewritten as
\begin{align}
	\frac
	{\mathcal{I}(R \setminus [-|t|,|t|])}
	{\mathcal{I}(R)}
	 & =
	\frac
	{\mathcal{I}(R \setminus [-|t|,|t|])}
	{\mathcal{I}(R \setminus [-|t|,|t|]) + \mathcal{I}(R \cap [-|t|,|t|])}
	\\
	\label{eq:transformed_term}
	 & =
	\frac{1}
	{
		1+
		\left(
		\mathcal{I}(R \cap [-|t|,|t|]) \middle/
		\mathcal{I}(R \setminus [-|t|,|t|])
		\right)
	}.
\end{align}
Moreover, from the fact that $R \cap [-|t|,|t|]$ and $R \setminus [-|t|,|t|]$ are disjoint and $f$ is always non-negative, we have the following four inequalities:
\begin{align}
	\mathcal{I}(R_i^\mathrm{inf} \cap [-|t|,|t|])
	 & =
	\mathcal{I}((R_i \cup S_i^c)\cap [-|t|,|t|])
	\\
	 & \geq \mathcal{I}(R  \cap [-|t|,|t|]),
	\\
	\mathcal{I}(R_i^\mathrm{inf} \setminus [-|t|,|t|])
	 & =
	\mathcal{I}(R_i\setminus [-|t|,|t|])
	\\
	 & \leq \mathcal{I}(R \setminus [-|t|,|t|]),
	\\
	\mathcal{I}(R_i^\mathrm{sup} \cap [-|t|,|t|])
	 & =
	\mathcal{I}(R_i\cap [-|t|,|t|])
	\\
	 & \leq \mathcal{I}(R \cap [-|t|,|t|]),
	\\
	\mathcal{I}(R_i^\mathrm{sup} \setminus [-|t|,|t|])
	 & =
	\mathcal{I}((R_i \cup S_i^c)\setminus [-|t|,|t|])
	\\
	 & \geq \mathcal{I}(R \setminus [-|t|,|t|]),
\end{align}
for any $R\in\mathcal{R}_i$,
and then the proof is completed.
\subsection{Proof of Lemma~\ref{lemm:bounds_property}}
\label{app:proof_of_bounds_property}
We will use the same notations as in the proof of Theorem~1 (Appendix~\ref{app:proof_of_main_theorem}).
We start by showing that the set family sequence $\left\{\mathcal{R}_i\right\}_{i\in\mathbb{N}}$ is a monotonically decreasing sequence and converges to the singleton $\left\{R_\infty\right\}$.
In fact, the monotonically decreasing property is shown from the formula transformation shown in
\begin{equation}
	R_i \subset R_{i+1}\subset R_{i+1}\cup S_{i+1}^c \subset R_\infty \cup S_{i+1}^c \subset
	R_i \cup S_i^c \cup S_{i+1}^c = R_i \cup S_i^c.
\end{equation}
Thus, the convergence of the set family sequence $\left\{\mathcal{R}_i\right\}_{i\in\mathbb{N}}$ is ensured, and its convergence destination is the singleton $\left\{R_\infty\right\}$ from $S_\infty^c=\mathbb{R}^c=\emptyset$ and its definition.
Then, combined with the fact that
\begin{equation}
	L_i =
	\inf_{R\in\mathcal{R}_i}
	\frac
	{\mathcal{I}(R\setminus [-|t|,|t|])}
	{\mathcal{I}(R)},\
	U_i =
	\sup_{R\in\mathcal{R}_i}
	\frac
	{\mathcal{I}(R\setminus [-|t|,|t|])}
	{\mathcal{I}(R)},
\end{equation}
and the definition of infimum and supremum, the monotonicity and convergence of $L_i$ and $U_i$ are obvious.
\subsection{Corollary of Theorem~\ref{thm:main}}
\label{app:corollary_of_main_theorem}
Theorem~\ref{thm:main} is extension of the conservative selective $p$-value considered in~\citep{jewell2022testing,chen2023more}, which corresponds to the upper bound in the following corollary.
\begin{corollary}
	If the searched intervals $S_i$ equals a single interval $[-|t|-\delta, |t|+\delta]$ with $\delta>0$ (corresponding to the case where the parametric search is performed so that the interval containing $[-|t|,|t|]$ is expanded for a linear selection event), the upper bound $U_i$ of the selective $p$-value is given by
	\begin{equation}
		U_i = \frac
		{\mathcal{I}((R_i\setminus [-|t|,|t|])\cup (-\infty, -|t|-\delta]\cup [|t|+\delta, \infty))}
		{\mathcal{I}(R_i\cup (-\infty, -|t|-\delta]\cup [|t|+\delta, \infty))}.
	\end{equation}
\end{corollary}
\begin{proof}
	It is obvious from Theorem~\ref{thm:main} by noting that $S_i^c=(-\infty, -|t|-\delta]\cup [|t|+\delta, \infty)$ and then $S_i^c\cap [-|t|,|t|]=\emptyset$.
\end{proof}
\subsection{Proof of Lemma~\ref{lemm:rephrase}}
\label{app:proof_of_rephrase}
To minimize $U_{i+1}-L_{i+1}$, it is sufficient to maximize $L_{i+1}$ and minimize $U_{i+1}$.
From the proof of Theorem~\ref{thm:main} (Appendix~\ref{app:proof_of_main_theorem}) and the transformation in~\eqref{eq:transformed_term}, we have
\begin{align}
	L_{i+1} & =
	\frac
	{\mathcal{I}(R_{i+1}^\mathrm{inf}\setminus [-|t|, |t|])}
	{\mathcal{I}(R_{i+1}^\mathrm{inf})}
	=
	\frac{1}
	{
		1+
		\left(
		\mathcal{I}(R_{i+1}^\mathrm{inf} \cap [-|t|,|t|]) \middle/
		\mathcal{I}(R_{i+1}^\mathrm{inf} \setminus [-|t|,|t|])
		\right)
	},
	\\
	U_{i+1} & =
	\frac
	{\mathcal{I}(R_{i+1}^\mathrm{sup}\setminus [-|t|, |t|])}
	{\mathcal{I}(R_{i+1}^\mathrm{sup})}
	=
	\frac{1}
	{
		1+
		\left(
		\mathcal{I}(R_{i+1}^\mathrm{sup} \cap [-|t|,|t|]) \middle/
		\mathcal{I}(R_{i+1}^\mathrm{sup} \setminus [-|t|,|t|])
		\right)
	}.
\end{align}
Then, to maximize $L_{i+1}$ and minimize $U_{i+1}$, it is sufficient to maximize the following two quantities shown in
\begin{gather}
	\mathcal{I}(R_{i+1}^\mathrm{inf}\setminus [-|t|, |t|])
	=
	\mathcal{I}(R_{i+1}\setminus [-|t|, |t|]),\\
	\mathcal{I}(R_{i+1}^\mathrm{sup}\cap [-|t|, |t|])
	=
	\mathcal{I}(R_{i+1}\cap [-|t|, |t|]),
\end{gather}
and to minimize the following two quantities shown in
\begin{gather}
	\mathcal{I}(R_{i+1}^\mathrm{inf}\cap [-|t|, |t|])
	=
	\mathcal{I}((R_{i+1}\cup S_{i+1}^c)\cap [-|t|, |t|]),\\
	\mathcal{I}(R_{i+1}^\mathrm{sup}\setminus [-|t|, |t|])
	=
	\mathcal{I}((R_{i+1}\cup S_{i+1}^c)\setminus [-|t|, |t|]).
\end{gather}
Since the integral regions of the two quantities to be maximized are disjoint and the integral regions of the two quantities to be minimized are disjoint,
to optimize these four integral quantities,
it is sufficient to maximize $\mathcal{I}(R_{i+1})$ and minimize $\mathcal{I}(R_{i+1}\cup S_{i+1}^c)$.
Here, since $R_{i+1}$ and $S_{i+1}^c$ are obviously disjoint, we have $\mathcal{I}(R_{i+1}\cup S_{i+1}^c)=\mathcal{I}(R_{i+1})+\mathcal{I}(S_{i+1}^c)$.
Therefore, to minimize this quantity, it is sufficient to minimize $\mathcal{I}(R_{i+1})$ and $\mathcal{I}(S_{i+1}^c)$.
To note that $\mathcal{I}(S_{i+1})=1-\mathcal{I}(S_{i+1}^c)$ holds, it is sufficient to construct $S_{i+1}$ from $S_i$ so that the two integral quantities $\mathcal{I}(R_{i+1})$ and $\mathcal{I}(S_{i+1})$ are both maximized.
\subsection{Extension to One-Sided Test}
\label{app:extension_to_one_sided}
%
Similarly to the two-sided selective $p$-value in~\eqref{eq:one_dimensional_selective_p_value}, the selective $p$-values for the left-tailed and the right-tailed tests are given with the same notation as in Theorem~\ref{thm:main} as:
\begin{align}
	p^\mathrm{selective}_\mathrm{left}
	 & =
	\mathbb{P}_{\mathrm{H}_0}
	\left(
	Z < t
	~\big|~
	Z \in R_\infty
	\right), \\
	p^\mathrm{selective}_\mathrm{right}
	 & =
	\mathbb{P}_{\mathrm{H}_0}
	\left(
	Z > t
	~\big|~
	Z\in R_\infty
	\right)  \\
	 & =
	1 - p^\mathrm{selective}_\mathrm{left},
\end{align}
respectively.
Theorem~\ref{thm:main}, Lemma~\ref{lemm:bounds_property} and Lemma~\ref{lemm:rephrase} can be extended to the one-sided test by replacing $[-|t|,|t|]$ with $[t, \infty)$ for the left-tailed test and $(-\infty, t]$ for the right-tailed test.
The proof can be done exactly in the same way as
Appendix~\ref{app:proof_of_main_theorem},
\ref{app:proof_of_bounds_property} and
\ref{app:proof_of_rephrase},
except that $[-|t|,|t|]$ is replaced by $[t, \infty)$ or $(-\infty, t]$.

\newpage
\section{Experimental Details}
\label{appB}
\subsection{SI for SFS}
\label{app:si_for_sfs}
SI for SFS was first studied in~\citep{tibshirani2016exact}.
In this study, the authors only considered the case with conditioning both on the selected features and history for computational tractability.
Later, in~\citep{sugiyama2021more}, SI for SFS without conditioning on the history was studied by using parametric programming-based SI.
The method proposed in~\citep{sugiyama2021more} is computationally expensive because it requires exhaustive search of the data space.
The results of these previous studies are summarized below, using the same notation style as in section~\ref{sec:sec2}.
\paragraph{Stepwise feature selection.}
We denote by $\bm{r}(\bm{D}, X_\mathcal{M})$ a residual vector obtained by regressing $\bm{D}$ onto $X_\mathcal{M}$ for a set of features $\mathcal{M}$, i.e.,
\begin{equation}
    \bm{r}(\bm{D}, X_\mathcal{M}) = (I_n-X_\mathcal{M}(X_\mathcal{M}^\top X_\mathcal{M})^{-1}X_\mathcal{M}^\top) \bm{D}.
\end{equation}
We denote the index of the feature selected at step $k$ as $j_k$ and the set of selected features up to step $k$ as $\mathcal{M}_k = \{j_1, \ldots, j_k\}$ for $k \in [K]$.
In SFS, at each step $k$, the feature $j_k$ is selected as
\begin{equation}
    \label{eq:sfs_criterion}
    j_k
    = \argmin_{j \in [p] \setminus \mathcal{M}_{k-1}}
    \left\|
    \bm{r}(\bm{D}, X_{\mathcal{M}_{k-1} \cup \{j\}})
    \right\|^2_2,
\end{equation}
where $\mathcal{M}_0=\emptyset$ and $\bm{r}(\bm{D}, \mathcal{M}_0)=\bm{D}$.
\paragraph{Selective inference.}
We denote $\mathcal{A}_{\mathrm{SFS}}$ as the ML algorithm that maps a response vector $\bm{D}$ to a set of $K$ selected features:
\begin{equation}
    \mathcal{A}_{\mathrm{SFS}}\colon \bm{D} \mapsto \mathcal{M}_K.
\end{equation}
Based on the discussions in section~\ref{subsec:pp_based_si},
to conduct the statistical tests in~\eqref{eq:hypotheses_sfs} and \eqref{eq:hypotheses_sfs_chi},
we need to identify the region $\mathcal{Z}$ in~\eqref{eq:completion_of_Z} that corresponds to the event that the set of the selected features is the same as the one in the observed data, i.e., $\mathcal{A}_\mathrm{SFS}(r) = \mathcal{A}_\mathrm{SFS}^\mathrm{obs}$.
We first consider the over-conditioning regions in
\begin{equation}
    \mathcal{Z}_\mathrm{SFS}^\mathrm{oc}(\bm{a}+\bm{b}z)
    =
    \left\{
    r\in\mathbb{R}:
    \mathcal{A}_\mathrm{SFS}(r) = \mathcal{A}_\mathrm{SFS}(z),
    \mathcal{S}_\mathrm{SFS}(r) = \mathcal{S}_\mathrm{SFS}(z)
    \right\},
\end{equation}
where $\mathcal{S}_\mathrm{SFS}$ is the sub-algorithm that outputs the \emph{history} of the SFS algorithm:
\begin{equation}
    \mathcal{S}_\mathrm{SFS}: \bm{D} \mapsto (\mathcal{M}_1, \mathcal{M}_2,\ldots,\mathcal{M}_K),
\end{equation}
i.e., the sequence of the sets of the selected features in $K$ steps.
From~\eqref{eq:sfs_criterion}, the over-conditioning region can be represented as an intersection of a set of \emph{quadratic inequalities}:
\begin{equation}
    \mathcal{Z}_\mathrm{SFS}^\mathrm{oc}(\bm{a}+\bm{b}z)
    =
    \bigcap_{k \in [K]}
    \bigcap_{j \notin \mathcal{M}_{k-1}^z}
    \left\{
    r \in \mathbb{R}:
    g_{k, j}^z(r) \leq 0
    \right\},
\end{equation}
where $\mathcal{M}_k^z$ indicates the set of selected features up to step $k$ when the data $\bm{a}+\bm{b}z$ is given to the SFS algorithm, and
$g_{k, j}^z(r)$ is a quadratic function defined as
\begin{equation}
    g_{k, j}^z(r)
    =
    \left\|\bm{r}(\bm{a}+\bm{b}r, X_{\mathcal{M}_{k-1}^z \cup \{j_k^z\}})\right\|_2^2
    -
    \left\|\bm{r}(\bm{a}+\bm{b}r, X_{\mathcal{M}_{k-1}^z \cup \{j^z\}})\right\|_2^2.
\end{equation}
Note that the superscript $z$ in $g$ specify that they depend on $\mathcal{A}_\mathrm{SFS}(z)$ and $\mathcal{S}_\mathrm{SFS}(z)$.
Therefore, the region $\mathcal{Z}$ in \eqref{eq:completion_of_Z} can be written as:
\begin{equation}
    \mathcal{Z}_\mathrm{SFS}
    =
    \bigcup_{z\in\mathbb{R}\mid \mathcal{M}_K^z=\mathcal{A}^\mathrm{obs}_\mathrm{SFS}}
    \bigcap_{k \in [K]}
    \bigcap_{j \notin \mathcal{M}_{k-1}^z}
    \left\{
    r \in \mathbb{R}:
    g_{k, j}^z(r) \leq 0
    \right\}.
\end{equation}
\subsection{SI for DNN-Driven Hypotheses}
\label{app:si_for_dnn}
Problem setups we considered in this study are similar to those in \citep{duy2022quantifying,miwa2023valid}.
%
In these studies, not only conditioning on pixels selected as salient regions but also conditioning on additional factors such as activation functions and other operations like max-pooling have been considered.
Then, parametric programming-based SI was used to remove those additional conditions.
The computational cost of these methods are expensive because they require exhaustive search of the data space.
The results of these previous studies are summarized below, using the same notation style as in section~\ref{sec:sec2}.
\paragraph{Selective inference.}
We denote $\mathcal{A}_{\mathrm{DNN}}$ as the ML algorithm that maps an image $\bm{D}$ to a set of salient and non-salient regions:
\begin{equation}
    \mathcal{A}_{\mathrm{DNN}}\colon \bm{D} \mapsto \{\mathcal{O}_{\bm{D}}, \mathcal{B}_{\bm{D}}\}.
\end{equation}
Based on the discussions in section~\ref{subsec:pp_based_si},
to conduct the statistical test in~\eqref{eq:hypotheses_dnn},
we need to identify the region $\mathcal{Z}$ in~\eqref{eq:completion_of_Z} that corresponds to the event that the set of the pixels in the salient region is the same as the one in the observed data, i.e., $\mathcal{A}_\mathrm{DNN}(r) = \mathcal{A}_\mathrm{DNN}^\mathrm{obs}$.
We first consider the over-conditioning regions in
\begin{equation}
    \mathcal{Z}_\mathrm{DNN}^\mathrm{oc}(\bm{a}+\bm{b}z)
    =
    \left\{
    r\in\mathbb{R}:
    \mathcal{A}_\mathrm{DNN}(r) = \mathcal{A}_\mathrm{DNN}(z),
    \mathcal{S}_\mathrm{DNN}(r) = \mathcal{S}_\mathrm{DNN}(z)
    \right\},
\end{equation}
where $\mathcal{S}_\mathrm{DNN}$ is the sub-algorithm that outputs a set of selected pieces of all the activation functions (AFs) in the DNN, i.e.,
\begin{equation}
    \label{equ:define_dnn_over_condition}
    \mathcal{S}_\mathrm{DNN}: \bm{D} \mapsto \{s_{\ell, h}\}_{\ell \in [L], h\in[H^\ell]},
\end{equation}
where
$L$ is the number of layers,
$H^{\ell}$ is the number of hidden nodes at layer $\ell^\mathrm{th}$, and
$s_{\ell, h}$ is $0$ if the activation function at the $h^\mathrm{th}$ node of layer $\ell^\mathrm{th}$ is \emph{inactive}, and $1$ otherwise.
By considering a class of DNN that consists of affine operations and piecewise-linear AFs (such as ReLU), the over-conditioning region can be represented as an intersection of a set of \emph{linear inequalities}:
\begin{equation}
    \mathcal{Z}_\mathrm{DNN}^\mathrm{oc}(\bm{a}+\bm{b}z)
    =
    \bigcap_{\ell \in [L]}
    \bigcap_{h \in [H^\ell]}
    \left\{
    r \in \mathbb{R}:
    \Omega_{s_{\ell, h}}^z (\bm{a}+\bm{b}r) \leq \bm{\omega}_{s_{\ell, h}}^z
    \right\},
\end{equation}
where the matrix $\Omega_{s_{\ell, h}}^z$ and the vector $\bm{\omega}_{s_{\ell, h}}^z$ are defined depending on whether $s_{\ell, h}$ is active or inactive when the data $\bm{a}+\bm{b}z$ is given to the sub-algorithm $\mathcal{S}_\mathrm{DNN}$.
Note that the superscript $z$ in $\Omega$ and $\bm{\omega}$ specify that they depend on $\mathcal{A}_\mathrm{DNN}(z)$ and $\mathcal{S}_\mathrm{DNN}(z)$.
Therefore, the region $\mathcal{Z}$ in~\eqref{eq:completion_of_Z} can be written as:
\begin{equation}
    \mathcal{Z}_\mathrm{DNN}
    =
    \bigcup_{z\in\mathbb{R} \mid \mathcal{A}(\bm{a}+\bm{b}z)=\mathcal{A}^\mathrm{obs}_\mathrm{DNN}}
    \bigcap_{\ell \in [L]}
    \bigcap_{h \in [H^\ell]}
    \left\{
    r \in \mathbb{R}:
    \Omega_{s_{\ell, h}}^z (\bm{a}+\bm{b}r) \leq \bm{\omega}_{s_{\ell, h}}^z
    \right\}.
\end{equation}
\begin{remark}
    (Activation functions and operations in a trained DNN) We focus on a trained DNN where the AFs used at hidden layers are piecewise linear.
    Otherwise, if there is any specific demand to use non-piecewise linear functions such as sigmoid functions, we can apply a piecewise-linear approximation approach to these functions.
    Some basic operations in a trained DNN are characterized by an intersection of linear inequalities.
    For example, the max-pooling can be written as:
    \begin{equation}
        v_1 = \max\{v_1, v_2, v_3\}
        \quad \Leftrightarrow \quad
        \bm{e}_1^\top \bm{v} \leq \bm{e}_2^\top \bm{v}
        \text{ and }
        \bm{e}_1^\top \bm{v} \leq \bm{e}_3^\top \bm{v},
    \end{equation}
    where $\bm{v}=(v_1, v_2, v_3)^\top$ and $\bm{e}_i$ is the standard basis vector with $1$ at position $i$.
    Furthermore, if the AFs are used at the output layer (e.g., sigmoid function), we need not perform the piecewise-linear approximation task, because we can define the set of linear inequalities based on the values prior to activation.
    In this way, when a DNN has a component represented by an intersection of linear inequalities, we can include all the events that each of those inequalities holds in the definition of $\mathcal{S}_\mathrm{DNN}$ in~\eqref{equ:define_dnn_over_condition}.
\end{remark}
\subsection{Details for Setup of Experiments}
\label{app:experiment_details}
We executed the all experiments on Intel(R) Xeon(R) CPU Gold 8338.
\paragraph{Network structure.}
In all the experiments of selective $z$-test in DNN, we used the network structure shown in Figure~\ref{fig:miwa-san-paper}.
\paragraph{Methods for comparison.}
We compared our proposed method with the following approaches:
\begin{itemize}
    \item \texttt{naive}: The classical test is used to compute the naive $p$-value.
    \item \texttt{OC}: This is a method of additionally conditioning on the observation of the sub-algorithm $\mathcal{S}$, and corresponds to the case where the search is terminated when $R_1$ is obtained in Procedure~\ref{alg:simple}.
    \item \texttt{exhaustive}: This is a method of PP-based SI with exhaustive search, i.e., it searches a certain interval, which is $[-20\sigma\|\bm{\eta}\|, 20\sigma\|\bm{\eta}\|]$ for the selective $z$-test and $[0, 100]$ for the selective $\chi$-test.
          Note that if the observed value $t$ of the test statistic is not included in these intervals, then $[-(10\sigma\|\bm{\eta}\|+|t|), 10\sigma\|\bm{\eta}\|+|t|]$ is searched for the selective $z$-test and $[0, 50+|t|]$ is searched for the selective $\chi$-test.
\end{itemize}
\begin{figure}[htbp]
    \centering
    \includegraphics[width=1.0\linewidth]{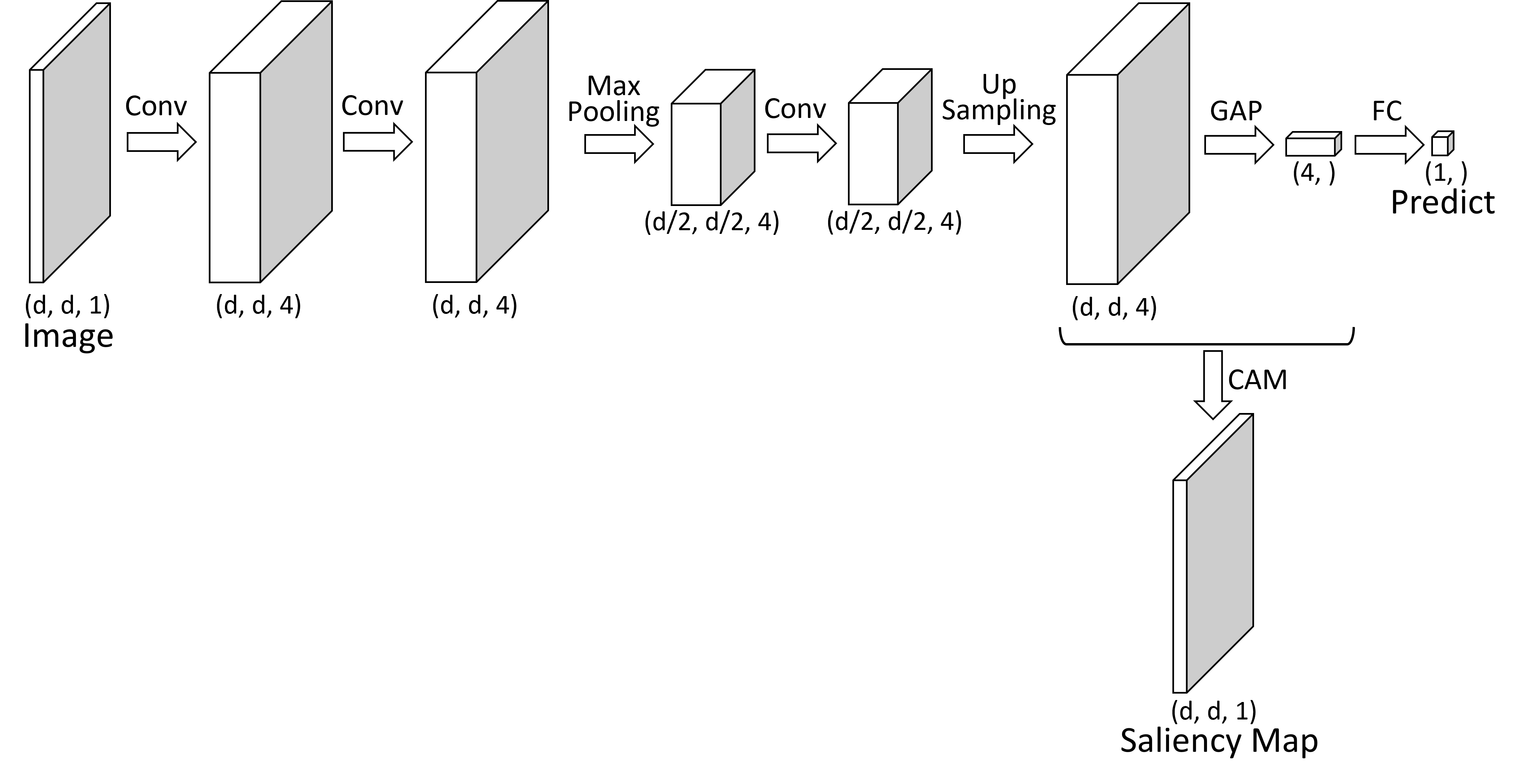}
    \caption{Network structure}
    \label{fig:miwa-san-paper}
\end{figure}

\newpage
\section{Extension to Selective Confidence Interval}
\label{app:extension_to_confidence_interval}
We extend our method to selective confidence intervals.
For the selective $z$-test, our method can be extended to guarantee the precision of selective confidence intervals.
However, for the selective $\chi$-test, to the best of our knowledge, no method has been proposed to compute selective confidence intervals.
Note that in \citep{yang2016selective}, a method has been proposed to compute conservative selective confidence intervals by assuming a stronger null hypothesis, to which our method can be applied similarly.
In this section, we only consider the selective $z$-test and redefine the test statistic as $T(\bm{D})=\bm{\eta}^\top \bm{D}/\sigma\|\bm{\eta}\|$ for simplicity (by this redefinition, the unconditional distribution of test statistic is standard normal).
\paragraph{Key idea.}
We use the same notation as in section~\ref{sec:sec3}, and explicitly denote integrals since we consider the multiple integrand functions.
Let $\mu_c$ be defined for any $c\in(0,1)$ to satisfy
\begin{equation}
    \label{eq:definition_of_mu_c}
    c =
    \frac
    {\int_{R_\infty\cap (-\infty, t]}\phi(z-\mu_c)dz}
    {\int_{R_\infty}\phi(z-\mu_c)dz},
\end{equation}
where $\phi$ is the probability density function of the standard normal distribution.
The existence and uniqueness of $\mu_c$ are obvious from the following lemma.
\begin{lemma}
    \label{lemm:cdf_monotonicity}
    For any $R\in\mathcal{B}(\mathbb{R})$, the cumulative distribution function of the truncated distribution of the normal distribution $\mathcal{N}(\mu, 1)$ in region $R$ is monotonically decreasing in $\mu$, i.e., for any $t\in \mathbb{R}$ and $\mu_1>\mu_2$,
    \begin{equation}
        \frac
        {\int_{R\cap (-\infty, t]}\phi(z-\mu_1)dz}
        {\int_{R}\phi(z-\mu_1)dz}
        <
        \frac
        {\int_{R\cap (-\infty, t]}\phi(z-\mu_2)dz}
        {\int_{R}\phi(z-\mu_2)dz}.
    \end{equation}
\end{lemma}
The proof of Lemma~\ref{lemm:cdf_monotonicity} is presented in Appendix~\ref{app:proof_of_cdf_monotonicity}.
The main idea of the proof is that the truncated normal distribution has monotone likelihood ratio in the mean parameter.
This lemma is key to prove the following theorem in this section.
\paragraph{Selective confidence interval.}
By using the $\mu_c$ defined in~\eqref{eq:definition_of_mu_c}, a $(1-\alpha)$ selective confidence interval for $\bm{\eta}^\top \bm{\mu}/\sigma \|\bm{\eta}\|$ is defined as
\begin{equation}
    [\mu_{1-\alpha/2}, \mu_{\alpha/2}].
\end{equation}
Then, our goal is to find the upper and lower bounds of $\mu_c$ for any $c\in(0,1)$ from $S_i$ and $R_i$.
Note that we consider the lower and upper bounds for each end of the confidence interval.
\begin{theorem}
    \label{thm:confidence_interval}
    For any $c\in(0,1)$ and $i\in\mathbb{N}$, $\mu_c$ is bounded by $\mu_c\in[\mu_{c,i}^l,\mu_{c,i}^u]$, which satisfy
    \begin{equation}
        \label{eq:condition_for_ci}
        c =
        \frac
        {\int_{R_i\cap (-\infty, t]}\phi(z-\mu_{c,i}^l)dz}
        {\int_{R_i\cup (S_i^c\cap [t, \infty))}\phi(z-\mu_{c,i}^l)dz}
        =
        \frac
        {\int_{(R_i\cup S_i^c)\cap (-\infty, t]}\phi(z-\mu_{c,i}^u)dz}
        {\int_{R_i\cup (S_i^c\cap (-\infty, t])}\phi(z-\mu_{c,i}^u)dz}.
    \end{equation}
\end{theorem}
The proof of Theorem~\ref{thm:confidence_interval} is presented in Appendix~\ref{app:proof_of_confidence_interval}.
In implementation, we can compute $\mu_{c,i}^l$ and $\mu_{c,i}^u$ by using the binary search.
The lower bounds $\mu_{c,i}^l$ and upper bounds $\mu_{c,i}^u$ of the $\mu_c$ in Theorem~\ref{thm:confidence_interval} have some reasonable properties described in the following lemma.
\begin{lemma}
    \label{lemm:ci_monotonicity}
    The lower bounds $\mu_{c,i}^l$ and upper bounds $\mu_{c,i}^u$ of $\mu_c$ in Theorem~\ref{thm:confidence_interval} are monotonically increasing and decreasing, in $i$, respectively.
    Furthermore, each of them converges to $\mu_c$ as $i\to\infty$.
\end{lemma}
The proof of Lemma~\ref{lemm:ci_monotonicity} is presented in Appendix~\ref{app:proof_of_ci_monotonicity}.
Lemma~\ref{lemm:ci_monotonicity} indicates that, in Procedure~\ref{alg:simple}, as the iteration progresses, the bounds of $\mu_c$ are become tighter and converge to $\mu_c$.
\paragraph{Applications.}
To apply Theorem~\ref{thm:confidence_interval}, we need to pay attention to two points.
First, the binary search must be performed to compute $\mu_{c,i}^l$ and $\mu_{c,i}^u$ in~\eqref{eq:condition_for_ci}.
This means that we need to perform the binary search four times (the lower and upper bounds for each end of the confidence interval) at each iteration $i$, which leads to the slightly high computational cost.
It occurs especially when the number of instances $n$ and the number of features $p$ in the SFS problem setup are small and then the over-conditioning region is computed with low cost.

Second, the bounds are very loose when the truncated intervals $R_i$ is small.
As an example, we consider the case where $t=1.244, R_i=[1.106, 1.351], S_i=[-20.0,20.0]$.
This is not at all rare case in the DNN problem setup and even smaller $R_i$ may be obtained.
In this case, if we consider the $0.95$ selective confidence interval $[\mu_{0.975}, \mu_{0.025}]$, the lower and upper bounds of each of them are computed as
\begin{equation}
    -25.09\leq \mu_{0.975}\leq -9.525,\ 10.80\leq \mu_{0.025}\leq 35.67,
\end{equation}
respectively, by Theorem~\ref{thm:confidence_interval}.
\subsection{Proof of Lemma~\ref{lemm:cdf_monotonicity}}
\label{app:proof_of_cdf_monotonicity}
We first show that the truncated normal distribution has monotone likelihood ratio in the mean parameter.
For any $z\in\mathbb{R}$, we have
\begin{equation}
    \frac
    {\phi(z-\mu_1)}
    {\phi(z-\mu_2)}
    =
    \exp\left(
    \frac{1}{2}(\mu_1-\mu_2)z - \frac{1}{2}(\mu_1^2-\mu_2^2)
    \right).
\end{equation}
Then, this ratio is monotonically increasing in $z$ from $\mu_1>\mu_2$ and we have
\begin{equation}
    \frac
    {\phi(z_1-\mu_1)}
    {\phi(z_1-\mu_2)}
    >
    \frac
    {\phi(z_2-\mu_1)}
    {\phi(z_2-\mu_2)},\
    z_1>z_2.
\end{equation}
This implies
\begin{equation}
    \phi(z_1-\mu_1)\phi(z_2-\mu_2)
    >
    \phi(z_1-\mu_2)\phi(z_2-\mu_1),\
    z_1>z_2.
\end{equation}
Therefore, by integrating both sides with respect to $z_2$ on $R\cap (-\infty, t]$ for $t<z_1$, we have
\begin{equation}
    \phi(z_1-\mu_1)
    \int_{R\cap (-\infty, t]}\phi(z-\mu_2)dz
    >
    \phi(z_1-\mu_2)
    \int_{R\cap (-\infty, t]}\phi(z-\mu_1)dz,\ t<z_1.
\end{equation}
Then, by integrating both sides with respect to $z_1$ on $R\cap (t, \infty)$, we have
\begin{gather}
    \left(
    \int_{R}\phi(z-\mu_1)dz
    -
    \int_{R\cap (-\infty, t]}\phi(z-\mu_1)dz
    \right)
    \int_{R\cap (-\infty, t]}\phi(z-\mu_2)dz \\
    >
    \left(
    \int_{R}\phi(z-\mu_2)dz
    -
    \int_{R\cap (-\infty, t]}\phi(z-\mu_2)dz
    \right)
    \int_{R\cap (-\infty, t]}\phi(z-\mu_1)dz.
\end{gather}
Thus, by dividing both sides by $\int_{R}\phi(z-\mu_1)dz\int_{R}\phi(z-\mu_2)dz$, we have
\begin{equation}
    \frac
    {\int_{R\cap (-\infty, t]}\phi(z-\mu_1)dz}
    {\int_{R}\phi(z-\mu_1)dz}
    <
    \frac
    {\int_{R\cap (-\infty, t]}\phi(z-\mu_2)dz}
    {\int_{R}\phi(z-\mu_2)dz},
\end{equation}
as required.
\subsection{Proof of Theorem~\ref{thm:confidence_interval}}
\label{app:proof_of_confidence_interval}
The existence and uniqueness of $\mu_{c,i}^l$ and $\mu_{c,i}^u$ are obvious from the Lemma~\ref{lemm:cdf_monotonicity}.
Let us consider Theorem~\ref{thm:main} for left-tailed $p$-value as discussed in Appendix~\ref{app:extension_to_one_sided}.
By computing the lower bound for the probability density function $\phi(z-\mu_{c,i}^l)$ and the upper bound for $\phi(z-\mu_{c,i}^u)$, we have
\begin{gather}
    c
    =
    \frac
    {\int_{R_i\cap (-\infty, t]}\phi(z-\mu_{c,i}^l)dz}
    {\int_{R_i\cup (S_i^c\cap [t, \infty))}\phi(z-\mu_{c,i}^l)dz}
    \leq
    \frac
    {\int_{R_\infty\cap (-\infty, t]}\phi(z-\mu_{c,i}^l)dz}
    {\int_{R_\infty}\phi(z-\mu_{c,i}^l)dz}, \\
    \frac
    {\int_{R_\infty\cap (-\infty, t]}\phi(z-\mu_{c,i}^u)dz}
    {\int_{R_\infty}\phi(z-\mu_{c,i}^u)dz}
    \leq
    \frac
    {\int_{(R_i\cup S_i^c)\cap (-\infty, t]}\phi(z-\mu_{c,i}^u)dz}
    {\int_{R_i\cup (S_i^c\cap (-\infty, t])}\phi(z-\mu_{c,i}^u)dz}
    =
    c,
\end{gather}
respectively.
Then, combined with the definition of $\mu_c$ in~\eqref{eq:definition_of_mu_c}, we have
\begin{equation}
    \frac
    {\int_{R_\infty\cap (-\infty, t]}\phi(z-\mu_{c,i}^u)dz}
    {\int_{R_\infty}\phi(z-\mu_{c,i}^u)dz}
    \leq
    \frac
    {\int_{R_\infty\cap (-\infty, t]}\phi(z-\mu_c)dz}
    {\int_{R_\infty}\phi(z-\mu_c)dz}
    \leq
    \frac
    {\int_{R_\infty\cap (-\infty, t]}\phi(z-\mu_{c,i}^l)dz}
    {\int_{R_\infty}\phi(z-\mu_{c,i}^l)dz},
\end{equation}
which indicates that $\mu_{c,i}^l\leq \mu_c \leq \mu_{c,i}^u$ from Lemma~\ref{lemm:cdf_monotonicity}.
\subsection{Proof of Lemma~\ref{lemm:ci_monotonicity}}
\label{app:proof_of_ci_monotonicity}
Let us consider Lemma~\ref{lemm:bounds_property} for left-tailed $p$-value as discussed in Appendix~\ref{app:extension_to_one_sided}.
By considering the lower bound for the probability density function $\phi(z-\mu_{c,i}^l)$ and the upper bound for $\phi(z-\mu_{c,i}^u)$, we have
\begin{gather}
    c
    =
    \frac
    {\int_{R_i\cap (-\infty, t]}\phi(z-\mu_{c,i}^l)dz}
    {\int_{R_i\cup (S_i^c\cap [t, \infty))}\phi(z-\mu_{c,i}^l)dz}
    \leq
    \frac
    {\int_{R_{i+1}\cap (-\infty, t]}\phi(z-\mu_{c,i}^l)dz}
            {\int_{R_{i+1}\cup (S_{i+1}^c\cap [t, \infty))}\phi(z-\mu_{c,i}^l)dz}, \\
    \frac
    {\int_{(R_{i+1}\cup S_{i+1}^c)\cap (-\infty, t]}\phi(z-\mu_{c,i}^u)dz}
    {\int_{R_{i+1}\cup (S_{i+1}^c\cap (-\infty, t])}\phi(z-\mu_{c,i}^u)dz}
    \leq
    \frac
    {\int_{(R_i\cup S_i^c)\cap (-\infty, t]}\phi(z-\mu_{c,i}^u)dz}
    {\int_{R_i\cup (S_i^c\cap (-\infty, t])}\phi(z-\mu_{c,i}^u)dz}
    =
    c,
\end{gather}
respectively.
Then, combined with the definition of $\mu_{c,i+1}^l$ and $\mu_{c,i+1}^u$ in~\eqref{eq:condition_for_ci}, we have
\begin{gather}
    \frac
    {\int_{R_{i+1}\cap (-\infty, t]}\phi(z-\mu_{c,i+1}^l)dz}
            {\int_{R_{i+1}\cup (S_{i+1}^c\cap [t, \infty))}\phi(z-\mu_{c,i+1}^l)dz}
    \leq
    \frac
    {\int_{R_{i+1}\cap (-\infty, t]}\phi(z-\mu_{c,i}^l)dz}
            {\int_{R_{i+1}\cup (S_{i+1}^c\cap [t, \infty))}\phi(z-\mu_{c,i}^l)dz}, \\
    \frac
    {\int_{(R_{i+1}\cup S_{i+1}^c)\cap (-\infty, t]}\phi(z-\mu_{c,i}^u)dz}
    {\int_{R_{i+1}\cup (S_{i+1}^c\cap (-\infty, t])}\phi(z-\mu_{c,i}^u)dz}
    \leq
    \frac
    {\int_{(R_{i+1}\cup S_{i+1}^c)\cap (-\infty, t]}\phi(z-\mu_{c,i+1}^u)dz}
    {\int_{R_{i+1}\cup (S_{i+1}^c\cap (-\infty, t])}\phi(z-\mu_{c,i+1}^u)dz},
\end{gather}
respectively.
These indicate that $\mu_{c,i}^l\leq \mu_{c,i+1}^l$ and $\mu_{c,i}^u\geq \mu_{c,i+1}^u$ from Lemma~\ref{lemm:cdf_monotonicity}.
Furthermore, by the definition of $\mu_{c,i}^l$ and $\mu_{c,i}^u$ in~\eqref{eq:condition_for_ci}, the convergence of $\mu_{c,i}^l$ and $\mu_{c,i}^u$ to $\mu_c$ is obvious from $R_i\to R_\infty$ and $S_i^c\to \emptyset$.

\newpage

\bibliographystyle{plainnat}
\bibliography{ref}

\end{document}